\documentclass[journal]{IEEEtran}
\usepackage[dvips]{graphicx}
\usepackage{epsfig,enumitem,multirow}
\usepackage{amsthm,amsmath,amssymb,mathrsfs,bm}
\usepackage{amsfonts,dsfont,color,bbm}
\usepackage{algorithm,algorithmic}
\usepackage[font=footnotesize,labelfont=bf]{caption}
\usepackage{subcaption,slashbox}
\usepackage{cite,epstopdf,epsf,epsfig}
\usepackage{resizegather}
\usepackage{booktabs}
\usepackage{float}
\usepackage{mathtools}
\usepackage{xr}
\usepackage{dsfont}
\usepackage{slashbox,url}
\usepackage{tikz}
\usetikzlibrary{positioning, fit, arrows.meta, shapes,matrix,calc,decorations.pathreplacing}

\pgfkeys{tikz/mymatrixenv/.style={decoration=brace,every left delimiter/.style={xshift=3pt},every right delimiter/.style={xshift=-3pt}}}
\pgfkeys{tikz/mymatrix/.style={matrix of math nodes,left delimiter=[,right delimiter={]},inner sep=2pt,column sep=1em,row sep=0.5em,nodes={inner sep=0pt}}}
\pgfkeys{tikz/mymatrixbrace/.style={decorate,thick}}

\tikzset{
    cheating dash/.code args={on #1 off #2}{
        \csname tikz@addoption\endcsname{%
            \pgfgetpath\currentpath%
            \pgfprocessround{\currentpath}{\currentpath}%
            \csname pgf@decorate@parsesoftpath\endcsname{\currentpath}{\currentpath}%
            \pgfmathparse{\csname pgf@decorate@totalpathlength\endcsname-#1}\let\rest=\pgfmathresult%
            \pgfmathparse{#1+#2}\let\onoff=\pgfmathresult%
            \pgfmathparse{max(floor(\rest/\onoff), 1)}\let\nfullonoff=\pgfmathresult%
            \pgfmathparse{max((\rest-\onoff*\nfullonoff)/\nfullonoff+#2, #2)}\let\offexpand=\pgfmathresult%
            \pgfsetdash{{#1}{\offexpand}}{0pt}}%
    }
}

\theoremstyle{break}

\newtheorem{theorem}{Theorem}
\newtheorem{remark}{Remark}
\newtheorem{lemma}{Lemma}
\theoremstyle{definition}

\usepackage{balance}

\newcommand{\Matrixinversionlemma}{( \textbf{A}-\textbf{B}\textbf{C}\textbf{D})^{-1} = \textbf{A}^{-1} + \textbf{A}^{-1} \textbf{B} ( \textbf{C}^{-1} - \textbf{D} \textbf{A}^{-1}\textbf{B})^{-1} \textbf{D} \textbf{A}^{-1} }

\newcommand{\T}{\bm{\theta}}
\newcommand{\Tt}{\bm{\theta}_t}
\newcommand{\Tht}{\bm{\hat{\theta}}_t}
\newcommand{\Thb}{\bm{\hat{\theta}}_1}

\newcommand{\Tit}{\bm{\theta}_{i,t}}
\newcommand{\Thit}{\bm{\hat{\theta}}_{i,t}}
\newcommand{\Titp}{\bm{\theta}_{i,t+1}}
\newcommand{\Thitp}{\bm{\hat{\theta}}_{i,t+1}}

\newcommand{\Hit}{\textbf{H}_{i,t}}
\newcommand{\Hjt}{\textbf{H}_{j,t}}

\newcommand{\Pit}{\textbf{P}_{i,t}}
\newcommand{\Pitp}{\textbf{P}_{i,t+1}}

\newcommand{\Piti}{\textbf{P}^{-1}_{i,t}}
\newcommand{\Pitpi}{\textbf{P}^{-1}_{i,t+1}}

\newcommand{\Kit}{\textbf{G}_{i,t}}

\newcommand{\Ht}{\textbf{H}_t}

\newcommand{\Mit}{\textbf{M}_{i,t}}
\newcommand{\Miti}{\textbf{M}^{-1}_{i,t}}
\newcommand{\Mitf}{(\Hit \Pit \Hit^T + r_{i,t} \I )}

\newcommand{\I}{\textbf{I}}

\newcommand{\Qt}{\textbf{Q}_t}

\newcommand{\Xit}{\bm{\zeta}_{i,t}}
\newcommand{\Xt}{\bm{\zeta}_t}

\newcommand{\Xo}{\overline{\zeta}}

\newcommand{\Rt}{\textbf{R}_t}

\newcommand{\zib}{\bm{\chi}_{i,1}}

\newcommand{\dt}{\textbf{d}_t}
\newcommand{\dht}{\bm{\hat{d}}_t}
\newcommand{\et}{\textbf{e}_t}

\newcommand{\normet}{\norm{\textbf{e}_t}}

\newcommand{\loss}{\norm{\dt-\dht}^2}

\newcommand{\zjt}{\bm{\chi}_{j,t}}

\newcommand{\zit}{\bm{\chi}_{i,t}}
\newcommand{\zitp}{\bm{\chi}_{i,t+1}}

\newcommand{\dhjt}{\bm{\hat{d}}_{j,t}}

\newcommand{\htT}{h_t( \{\textbf{x}_t\} ; \T)}
\newcommand{\htTt}{h_t( \{\textbf{x}_t\} ; \Tt)}
\newcommand{\htTht}{h_t( \{\textbf{x}_t\}; \Tht)}

\newcommand{\Delvt}{\Delta V_{i,t}}
\newcommand{\Vb}{\bm{\chi}^T_{i,1} \textbf{P}_{i,1}^{-1} \bm{\chi}^T_{i,1}}

\newcommand{\Tr}{\textrm{Tr}}

\DeclarePairedDelimiter{\norm}{\lVert}{\rVert}%

\begin{document}
\renewcommand{\thepage}{}

\title{An Efficient and Effective Second-Order Training Algorithm For LSTM-based Adaptive Learning}

\author{N. Mert Vural, Salih Erg\"{u}t and Suleyman S. Kozat, \textit{Senior Member, IEEE}
\thanks{This work is supported in part by the Turkish Academy of Sciences Outstanding Researcher Programme and TUBITAK Contract No: 117E153.}
\thanks{ N. M. Vural and S. S. Kozat are with the Department of Electrical and Electronics Engineering, Bilkent University, Ankara 06800, Turkey, e-mail: vural@ee.bilkent.edu.tr, kozat@ee.bilkent.edu.tr.}
\thanks{S. Erg\"{u}t is with the Turkcell Technology, 5G R\&D Team, Istanbul, Turkey, e-mail: salih.ergut@turkcell.com.tr}
\thanks{S. S. Kozat is also with the DataBoss A.S., email: serdar.kozat@data-boss.com.tr.}
}

\maketitle
\begin{abstract}
We study adaptive (or online) nonlinear regression with Long-Short-Term-Memory (LSTM) based networks, i.e., LSTM-based adaptive learning. In this context, we introduce an efficient Extended Kalman filter (EKF) based second-order training algorithm.  Our algorithm  is truly online, i.e., it does not assume any underlying data generating process and future information, except that the target sequence is bounded. Through an extensive set of experiments, we demonstrate significant performance gains achieved by our algorithm with respect to the state-of-the-art methods. Here, we mainly show that our algorithm consistently provides 10 to 45\% improvement in the accuracy compared to the widely-used adaptive methods Adam, RMSprop, and DEKF, and comparable performance to EKF with a 10 to 15 times reduction in the run-time.
\end{abstract}

\begin{center} \bfseries EDICS Category: MLR-SLER, MLR-DEEP.  \end{center}

\section{Introduction}
\subsection{Preliminaries}
We investigate adaptive (or online) learning, which has been extensively studied due to its applications in a wide set of problems, such as signal processing~\cite{CBianchi2006, Kivinen04, Kozat07}, neural network training~\cite{King14}, and algorithmic learning theory~\cite{Duchi11}.   In this problem, a learner (or the adaptive learning algorithm) is tasked with predicting the next value of a target sequence based on its knowledge about the previous input-output pairs~\cite{CBianchi2006}. For this task, commonly, nonlinear approaches are employed in the literature since the linear modeling is inadequate for a broad range of applications because of the linearity constraints~\cite{Kivinen04}. 

For adaptive learning, there exists a wide range of nonlinear approaches in the fields of signal processing and machine learning~\cite{Kivinen04, Engel2004}. However, these approaches usually suffer from prohibitively excessive computational requirements, and may provide poor performance due to overfitting and stability issues~\cite{Liu10}. Adopting neural networks is another method for adaptive nonlinear regression due to their success in approximating nonlinear functions. However, neural network-based regression algorithms have been shown to exhibit poor performance in certain applications~\cite{Liu10}. To overcome the limitations of those rather shallow networks, neural networks composed of multiple layers, i.e., deep neural networks (DNNs), have recently been introduced. In DNNs, each layer performs a feature extraction based on the previous layers, enabling them to model highly nonlinear structures~\cite{KriSut12}. However, this layered structure usually performs poorly in capturing the time dependencies, which are commonly encountered in adaptive regression problems~\cite{DeepRec13}. And thus, DNNs provide only limited performance in adaptive learning applications. 

To remedy this issue,  recurrent neural networks (RNNs) are used, as these networks have a feed-back connection that enables them to store past information. However, basic RNNs lack control structures, where the long-term components cause either exponential growth or decay in the norm of gradients~\cite{Bengio94}. Therefore, they are insufficient to capture long-term dependencies, which significantly restricts their performance in real-life applications. In order to resolve this issue, a novel RNN architecture with several control structures, i.e., Long-Short-Term-Memory network (LSTM), was introduced~\cite{lstm}. In this study, we are particularly interested in adaptive nonlinear regression with the LSTM-based networks due to their superior performance in capturing long-term dependencies.

For RNNs (including LSTMs), there exists a wide range of adaptive training methods to learn the network parameters~\cite{King14,Ruder16, lstmdekf}. Among them, first-order gradient-based methods are widely preferred due to their efficiency~\cite{Ruder16}. However,  in general, first-order techniques yield inferior performance compared to second-order techniques~\cite{lstmdekf,Martens16}, especially in applications where network parameters should be rapidly learned, e.g., when data is relatively scarce or highly non-stationary, as in most adaptive signal processing applications~\cite{Haykin}. As a second-order technique, the extended Kalman filter (EKF) learning algorithm has often been favored for its accuracy and convergence speed~\cite{Haykin}. However, the EKF learning algorithm has a quadratic computational requirement in the parameter size, which is usually prohibitive for practical applications due to the high number of parameters in modern networks, such as LSTMs~\cite{lstm}. 

To reduce the time complexity of EKF, it is common practice to approximate the state covariance matrix in a block-diagonal form by neglecting the covariance terms between the weights belonging to different nodes~\cite{Shah92}. We refer to this method as the Independent EKF (IEKF) method  since each neural node in LSTM in this case is assumed as an independent subsystem. By using IEKF, the computational requirement of EKF is reduced by the number of neural nodes in the network. Note that this computational saving  is especially beneficial to second-order LSTM optimization, since LSTMs contain a large number of nodes in practice -- usually more than $50$~\cite{lstmdekf, Coskun17, Tolga18, Yang17}. On the other hand, IEKF generally performs poorly compared to EKF since it ignores the correlation between most of the network weights~\cite{Haykin}.

In this study, we introduce an \emph{efficient and effective} second-order training algorithm that reduces the performance difference between EKF and IEKF. To develop our algorithm, we use the online learning approach~\cite{Zinkevich03}, i.e., we build a procedure that is guaranteed to predict the time series with performance close to that of the best predictor in a given set of models, without any assumption on the distribution on the observed time series. Therefore, our algorithm is highly practical in the sense that it does not assume any underlying data generating process or future information, except that the target sequence is bounded. Through an extensive set of experiments, we demonstrate significant performance gains achieved by our algorithm with respect to the state-of-the-art methods. Here, we mainly show that our algorithm provides $10$ to $45\%$ improvement in the accuracy of the widely-used adaptive methods Adam~\cite{King14}, RMSprop~\cite{Tieleman12}, and DEKF~\cite{Pusk94}, and comparable performance to EKF~\cite{Haykin} in $10$ to $15$ times smaller training time.

\subsection{Prior Art and Comparison}
In the neural network literature, Kalman filtering is commonly used to refer to inferring the latent (or hidden) variables from sequential observation~\cite{Haykin}. This task is especially relevant for the purposes of generative sequential learning~\cite{Krishnan15,Krishnan17,Gao16,Maddison17} and adaptive RNN (including LSTM) training~\cite{lstmdekf,Haykin, Tolga18}. In both generative sequential learning and adaptive RNN training, sequential data are usually assumed to be noisy observations of a nonlinear dynamical system, commonly modeled as RNNs. The main difference in those fields is their latent variable selections (see \cite[Section 4]{Krishnan15} and \cite[Section 3B]{Tolga18}). In the generative sequential learning studies, the hidden state vector of the model is often chosen as the latent variable~\cite{Krishnan15,Krishnan17,Gao16,Maddison17}. Since the likelihood of the observations become intractable in this case, variational methods (combined with either sampling-based~\cite{Maddison17} or first-order gradient methods~\cite{Krishnan15,Krishnan17,Gao16})  are used to learn the network weights.  On the other hand, in the adaptive RNN training studies, the network weights are chosen as the latent variable, and variants of EKF are used to find the locally optimum weights adaptively (or to train the RNN model online)~\cite{lstmdekf,Haykin, Tolga18}. We note that our study is an example of the latter.

Various optimization algorithms in the deep learning literature, such as Adam~\cite{King14}, RMSprop~\cite{Tieleman12} or SGD~\cite{Ruder16}, can be used to train LSTMs online. Among the widely used first-order algorithms,  the adaptive learning methods, e.g., Adam and RMSprop, provide faster convergence than the plain SGD~\cite{Sutskever13}.  The performance gains of Adam and RMSprop are justified with their approximate Hessian-based preconditioning, which let them use the second-order properties of the error surface to improve the convergence speed~\cite{Dauphin15}. However, the approximation employed in the first-order adaptive methods considers only the diagonal elements of the Hessian matrix, which severely limits their performance in comparison to the second-order algorithms that utilize a full Hessian matrix estimate~\cite{Martens16}.

For LSTMs, the Hessian-Free and EKF algorithms are capable of utilizing a full Hessian matrix estimate with reasonable time complexity~\cite{Martens10,Haykin}. However, the Hessian-Free algorithm requires a large size of batches to approximate the Hessian matrix, which makes it impractical for online processing. On the other hand, the EKF learning algorithm is highly appropriate for adaptive learning, since it recursively updates its Hessian estimate, i.e., its state covariance matrix, without using mini-batch statistics. Moreover, EKF is extensively studied in the neural network literature, where it has been repeatedly demonstrated to have faster convergence and better accuracy than the first-order methods~\cite{lstmdekf}. However, EKF requires a quadratic computational complexity in the number of parameters, which is prohibitive for most of the practical applications using LSTMs.

As noted earlier, it is common in practice to use a block-diagonal approximation for the state covariance matrix to reduce the computational complexity of EKF, which we referred to as the IEKF approach. The algorithms using the IEKF approach, such as the Decoupled EKF (DEKF)~\cite{Pusk94} or  Multiple Extended Kalman Algorithm (MEKA)~\cite{Shah92}, have been shown to provide better performance than the first-order methods with an acceptable reduction in the error performance (due to their computational savings) compared to EKF~\cite{Shah92}. Despite their empirical success, the existing methods are heuristic-based; hence, they have no guarantee to converge to an optimum set of weights during the training~\cite{Pusk94}. Moreover, due to the neglected covariance terms, they are sensitive to the hyperparameter selection and initialization, which degrades their performance compared to EKF~\cite{Haykin}.  

In this study, we are interested in reducing the performance difference between EKF and IEKF by providing an efficient and effective second-order training algorithm for LSTM-based adaptive learning. To this end, we introduce an IEKF-based algorithm with an adaptive hyperparameter selection strategy that guarantees errors to converge to a small interval.

There are only a few recent works studying convergence analysis of the RNN training with EKF~\cite{Rubio07,Wang11} and all these studies consider basic RNN training with full Hessian approximation, which cannot be easily extended to LSTMs due to the high computational complexity of EKF. Hence, unlike the previous works, we use the IEKF approach and introduce an IEKF-based training algorithm with a performance guarantee. To the best of our knowledge, our paper is the first study to provide a theoretically justified second-order algorithm for LSTM-based adaptive learning. Endowed with a theoretical guarantee, our algorithm performs very closely to EKF in a considerably shorter running time. As noted earlier, our algorithm does not require a priori knowledge of the environment for its performance guarantee (i.e., truly online), which makes it practical for a wide range of adaptive signal processing applications, such as time-series prediction~\cite{lstmdekf} and object tracking~\cite{Coskun17}.

\subsection{Contributions}
Our main contributions are as follows: 

\begin{enumerate}
\item To the best of our knowledge, we, as the first time for the literature, introduce a second-order training algorithm with a performance guarantee for LSTM-based adaptive learning algorithm.  
\item Since we construct our algorithm with the IEKF approach, our algorithm provides significant computational savings in comparison to the state-of-the-art second-order optimization methods~\cite{Martens16, Haykin}.
\item Our algorithm can be used in a broad range of adaptive signal processing applications, as it does not assume any underlying data generating process or future information, except that the target sequence is bounded~\cite{lstmdekf, Coskun17}.
\item Through an extensive set of experiments involving synthetic and real data, we demonstrate significant performance gains achieved by the proposed algorithm with respect to the state-of-the-art algorithms.
\item In the experiments, we particularly show that our algorithm provides a $10$ to $45\%$ reduction in the mean squared error performance compared to the widely-used adaptive learning method, i.e., Adam~\cite{King14}, RMSprop~\cite{Tieleman12}, and DEKF~\cite{Pusk94}, and comparable performance to EKF~\cite{Haykin} with a $10$ to $15$ times reduction in the run-time.
\end{enumerate}

\subsection{Organization of the Paper}
This paper is organized as follows. In Section~\ref{sec:model}, we formally introduce the adaptive regression problem and describe our LSTM model. In Section~\ref{sec:review}, we demonstrate the EKF and IEKF learning algorithms, where we also compare their computational requirements to motivate the reader for the analysis in the following section. In Section~\ref{sec:main}, we  develop  a truly online IEKF-based LSTM training algorithm with a performance guarantee. In Section~\ref{sec:exp}, we demonstrate the performance of our algorithm via an extensive set of experiments. We then finalize our paper with concluding remarks in Section~\ref{sec:concl}.

\section{Model and Problem Description}\label{sec:model}
\begin{figure}[t]
\centering
  \includegraphics[width=0.45\textwidth]{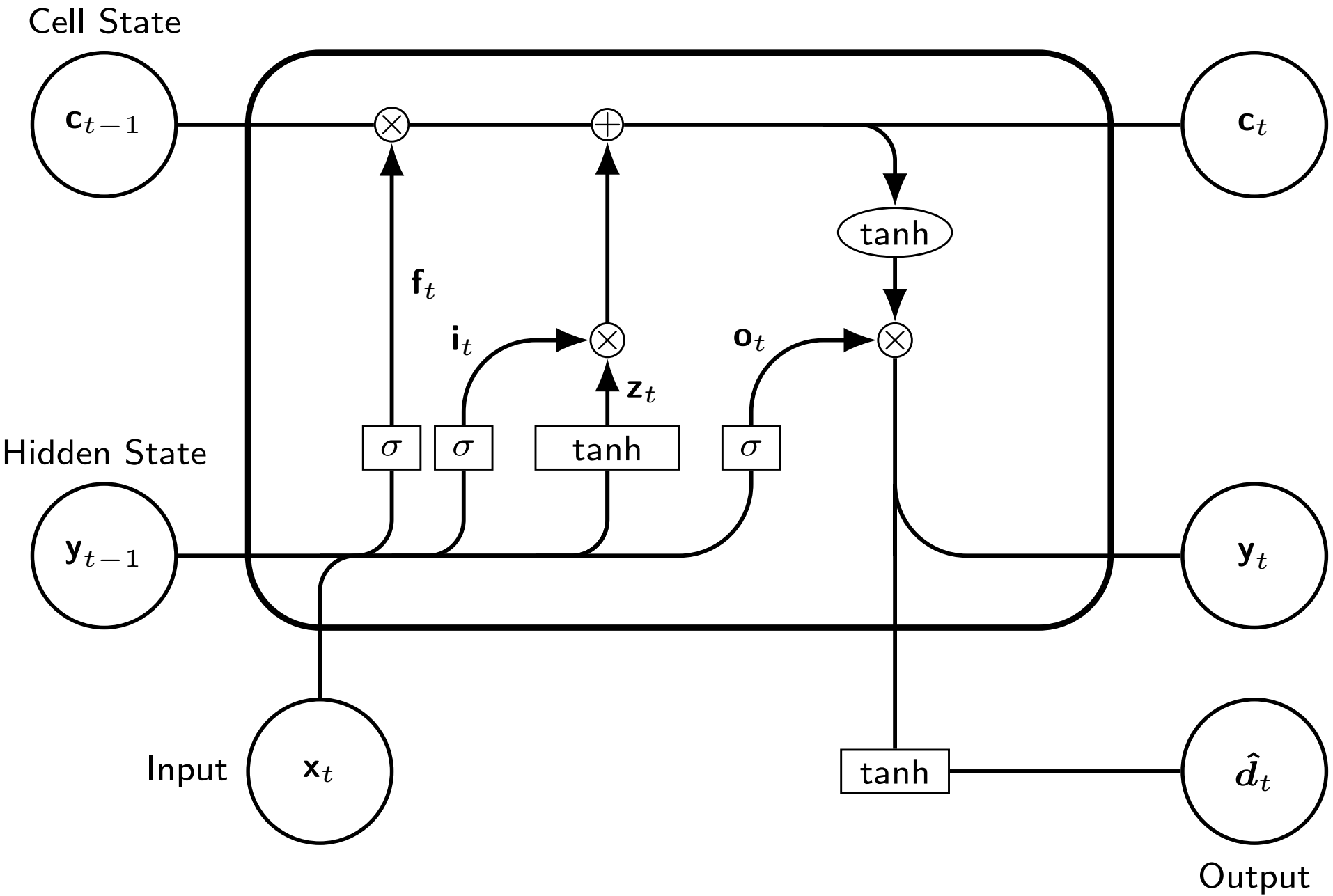}\\
 \caption{The detailed schematic of the equations given in (\ref{lstm_1})-(\ref{lstm_7}).}
 \label{fig:lstm}
\end{figure}
All vectors are column vectors and denoted by boldface lower case letters. Matrices are represented by boldface capital letters.  The $\textbf{0}$ (respectively $\textbf{1}$) represents a matrix or a vector of all zeros (respectively ones), whose dimensions are understood from on the context. $\textbf{I}$ is the identity matrix, whose dimensions are understood from the context. $\norm{\cdot}$ and $\Tr(\cdot)$ denote the Euclidean norm and trace operators. Given two matrices $\textbf{A}$ and $\textbf{B}$, $\textbf{A} > \textbf{B}$ (respectively $\geq$) means that $(\textbf{A}-\textbf{B})$ is a positive definite (respectively semi-positive definite) matrix. Given two vectors $\textbf{x}$ and $\textbf{y}$, $[\textbf{x} ; \textbf{y}]$ is their vertical concatenation. We use bracket notation $[n]$ to denote the set of the first $n$
positive integers, i.e., $[n]= \{ 1, \cdots, n \}$.

We define the adaptive regression problem as follows: We sequentially receive bounded target vectors, $\{\dt\}_{ t \geq 1}$, $\dt \in  [-1,1]^{n_d}$, and input vectors, $\{\textbf{x}_t\}_{t \geq 1}$, $\textbf{x}_t \in \mathbbm{R}^{n_x}$ such that our goal is to estimate $\dt$ based on our current and past observations $\{\cdots, \textbf{x}_{t-1}, \textbf{x}_t\}$.\footnote{We assume $\dt \in  [-1,1]^{n_d}$ for notational simplicity; however, our derivations hold for any bounded desired data sequence after shifting and scaling in magnitude.} Given our estimate $\dht$, which can only be a function of $\{\cdots, \textbf{x}_{t-1}, \textbf{x}_t\}$ and $\{\cdots, \textbf{d}_{t-2}, \textbf{d}_{t-1}\}$, we suffer a loss $\ell(\dt,\dht)$. The aim is to optimize the network with respect to the loss function $\ell(\cdot,\cdot)$. In this study, we particularly work with the squared error, i.e., $\ell(\dt,\dht)= \loss$. We note that since we observe the target value $\dt$ at each time step, i.e., full information setting,  all of our results hold \emph{in the deterministic sense}. We additionally note that our work can be extended to a wide range of cost functions (including the cross-entropy) using the analysis in~\cite[Section 3]{Pusk97}.  

In this paper, we study adaptive regression with LSTM-based networks due to their succes in modelling highly nonlinear sequential tasks~\cite{lstm}. As illustrated in Fig.~\ref{fig:lstm}, we use the most widely used LSTM model, where the activation functions are set to the hyperbolic  tangent function and the peep-hole connections are eliminated. As in Fig.~\ref{fig:lstm}, we use a single hidden layer based on the LSTM structure,  and an output layer with the hyperbolic tangent function\footnote{We use the hyperbolic tangent function in the output layer to ensure that the model estimates also lay in $[-1,1]^{n_d}$. However, our results hold for any output activation function, given that the function is differentiable, and its range is sufficient large to include the target vectors.}. Hence, we have:
\begin{align}
\textbf{z}_t&=\tanh(\textbf{W}^{(z)}[\textbf{x}_t ; \textbf{y}_{t-1} ]) \label{lstm_1}\\
\textbf{i}_t&= \sigma(\textbf{W}^{(i)}[\textbf{x}_t ; \textbf{y}_{t-1} ]) \\
\textbf{f}_t&= \sigma(\textbf{W}^{(f)}[\textbf{x}_t ; \textbf{y}_{t-1} ]) \\
\textbf{c}_t&= \textbf{i}_t \odot \textbf{z}_t + \textbf{f}_t \odot \textbf{c}_{t-1} \label{lstm:c} \\
\textbf{o}_t&= \sigma(\textbf{W}^{(o)}[\textbf{x}_t ; \textbf{y}_{t-1} ]) \label{lstm:5} \\
\textbf{y}_t&= \textbf{o}_t \odot \tanh(\textbf{c}_t) \label{lstm_6} \\
\dht&= \tanh(\textbf{W}^{(d)}  \textbf{y}_t) \label{lstm_7} 
\end{align}
where $\odot$ denotes the element-wise multiplication, $\textbf{c}_t \in \mathbbm{R}^{n_s}$ is the state vector, $\textbf{x}_t \in \mathbbm{R}^{n_x}$ is the input vector, and $\textbf{y}_t \in \mathbbm{R}^{n_s}$ is the output vector, and $\dht \in [-1,1]^{n_d}$ is our final estimation. Furthermore, $\textbf{i}_t$, $\textbf{f}_t$ and $\textbf{o}_t$ are the input, forget and output gates respectively. The sigmoid function $\sigma(.)$ and the hyperbolic tangent function $\tanh(.)$ applies point wise to the vector elements. The weight matrices are $\textbf{W}^{(z)}, \textbf{W}^{(i)}, \textbf{W}^{(f)}, \textbf{W}^{(o)} \in \mathbbm{R}^{n_s \times (n_x + n_s )}$ and $\textbf{W}^{(d)} \in \mathbbm{R}^{ n_d \times n_s}$. We note that although we do not explicitly write the bias terms, they can be included in (\ref{lstm_1})-(\ref{lstm_7}) by augmenting the input vector with a constant dimension.

\section{EKF-Based Adaptive Training Algorithms}\label{sec:review}
In this section, we introduce the EKF and IEKF learning algorithms within the LSTM-based adaptive learning framework. We note that in the neural network literature, there are two approaches to derive EKF-based learning algorithms: the parallel EKF approach, where both the network weights and hidden state vectors are treated as the states to be estimated by EKF, and the parameter-based EKF approach, where only the network weights are viewed as states to be estimated~\cite{Pusk94}. In the following, we derive our algorithms by using \emph{the parameter-based EKF approach}. We prefer to use the parameter-based EKF approach since, unlike the parallel EKF approach, the parameter-based EKF approach allows us to use the Truncated Backpropagation Through Time algorithm~\cite{Will95} to approximate the derivatives efficiently~\cite{Pusk94}. However, we emphasize that it is possible to adapt our analysis to the parallel EKF approach by using the state-space representation in \cite{Tolga18} and changing our mathematical derivations accordingly.

For notational convenience in the following derivations, we introduce two new notations: \textit{1)} We group all the LSTM parameters, i.e., $\textbf{W}^{(z)}, \textbf{W}^{(i)}, \textbf{W}^{(f)}, \textbf{W}^{(o)} \in \mathbbm{R}^{n_s \times (n_x + n_s )}$ and $\textbf{W}^{(d)} \in \mathbbm{R}^{n_d \times n_s}$, into a vector $\bm{\theta} \in \mathbbm{R}^{n_{\theta}}$, where $n_{\theta}= 4n_s(n_s+n_x) + n_s n_d$. \textit{2)} We use $\{\textbf{x}_t\}$ to denote the input sequence up to time $t$, i.e.,  $\{\textbf{x}_t\}=\{\textbf{x}_1,\textbf{x}_2,\cdots, \textbf{x}_t\}$.

Now, we are ready to derive the EKF and IEKF learning algorithms. 
 
\subsection{Adaptive Learning with EKF}
\begin{figure*}[t!]
\hspace{10mm}
\begin{tikzpicture}[
    font=\sf \scriptsize,
    >=LaTeX,
    cell/.style={
        rectangle, 
        rounded corners=1mm, 
        draw,
        very thick,
        },
    invcell/.style={
        rectangle, 
        rounded corners=1mm, 
        draw=none,
        },    
    operator/.style={
        circle,
        draw,
        inner sep=-0.5pt,
        minimum height =.2cm,
        },
    function/.style={
        ellipse,
        draw,
        inner sep=1pt
        },
    ct/.style={
        },
    gt/.style={
        rectangle,
        draw,
        thick,
        minimum width=4mm,
        minimum height=3mm,
        inner sep=1pt
        },
    mylabel/.style={
        font=\scriptsize\sffamily
        },
    ArrowC1/.style={
        rounded corners=.25cm,
        thick,
        },
    ArrowC2/.style={
        rounded corners=.5cm,
        thick,
        },
    ]
	
    \node [cell, minimum height =1.75cm, minimum width=2.5cm, align=center] (cell1) at (-7,0){LSTM cell \\ parametrized by $\Tt$ \\ \tiny{(The equations are} \\ \tiny{given in (1)-(6).)} } ;
    \node [cell, minimum height =1.75cm, minimum width=2.5cm, align=center] (cell2) at (-3.75,0){LSTM cell \\ parametrized by $\Tt$} ;
    \node [invcell, minimum height =1.75cm, minimum width=2.5cm, align=center] (cell3) at (-0.5,0){\Large{$\cdots$}} ;
    \node [cell, minimum height =1.75cm, minimum width=2.5cm, align=center] (cell4) at (2.75,0){LSTM cell \\ parametrized by $\Tt$} ;
    
    \node[ct] (x1) at (-7,-1.5) {$\textbf{x}_{1}$};
    \node[ct] (x2) at (-3.75,-1.5) {$\textbf{x}_{2}$};
    \node[ct] (xt) at (2.75,-1.5) {$\textbf{x}_{t}$};
    
    \node[ct] (c0) at (-9,0.3) {$\textbf{c}_{0}$};
    \node[ct] (y0) at (-9,-0.3) {$\textbf{y}_{0}$};

    \node [gt, minimum width=1 cm] (out) at (4.375,1.3) {tanh};
    
    \node[ct] (dt) at (4.375,2) {$\textbf{d}_{t}$};

	\draw[decoration={brace,mirror,raise=5pt},decorate,thick] (-7.15, -1.2) -- node[left=5pt,align=center] {Input \\ Vectors} (-7.15,- 1.7);  	
 	\draw[decoration={brace,mirror,raise=5pt},decorate,thick]  (-9.05,0.4) -- node[left=5.05pt,align=center] {Initial \\ States} (-9.05,-0.4); 
	\draw[decoration={brace,mirror,raise=5pt},decorate,thick] (-8.25,-1.65) -- node[below=8.5pt,align=center] {Unfolded version of the LSTM model in (\ref{lstm_1})-(\ref{lstm_7}) \\ over all the time steps up to the current time step $t$. \\ Note that all forward passes are parametrized by $\Tt$.} (4.375,-1.65); 
    \draw[decoration={brace,raise=5pt},decorate,thick] (4.75,1.65) --  node[right=5.1pt,align=center] {Output Layer \\ parametrized \\ by $\Tt$ \\ \tiny{(The equation is} \\ \tiny{given in (7).)} } (4.75, 0.9); 
    \draw[decoration={brace,mirror,raise=5pt},decorate,thick] (4.22,2.2) -- node[left=5.1pt,align=center] {Desired \\ data} (4.22,1.8);

	\draw [->, thick] (x1.north) -| (cell1.south);
	\draw [->, thick] (x2.north) -| (cell2.south);
	\draw [->, thick] (xt.north) -| (cell4.south);
 
 	\draw [->, thick] (c0.east) -- (c0-|cell1.west);
	\draw [->, thick] (y0.east) -- (y0-|cell1.west);
	
	\draw [->, thick] (c0-|cell1.east) -- node[above] {$\textbf{c}_1$} (c0-|cell2.west);
	\draw [->, thick] (y0-|cell1.east) -- node[below] {$\textbf{y}_1$}  (y0-|cell2.west);
	
	\draw [->, thick] (c0-|cell2.east) -- node[above] {$\textbf{c}_2$} (c0-|cell3.west);
	\draw [->, thick] (y0-|cell2.east) -- node[below] {$\textbf{y}_2$} (y0-|cell3.west);
	
	\draw [->, thick] (c0-|cell3.east) -- node[above] {$\textbf{c}_{t-1}$} (c0-|cell4.west);
	\draw [->, thick] (y0-|cell3.east) -- node[below] {$\textbf{y}_{t-1}$} (y0-|cell4.west);
	
	\draw [->, thick] (y0-|cell4.east) -| node[right] {$\textbf{y}_{t}$} (out.south);
	\draw [->, thick] (out.north) -- (dt.south);

\end{tikzpicture}  \\
 \caption{The detailed schematic of $\htTt$ given in (\ref{lstm_ss1}). The figure represents the unfolded version of the LSTM model in (\ref{lstm_1})-(\ref{lstm_7}) over all the time steps up to the current time step $t$. Here, the iterations start with the predetermined initial states $\textbf{y}_0$ and $\textbf{c}_0$, which are independent of the network weights. Then, the same LSTM forward pass (given in (\ref{lstm_1})-(\ref{lstm_6})) is repeatedly applied to the input sequence $\{\textbf{x}_t\}$, where the LSTM cells are parametrized by their corresponding weights in $\Tt$. Finally, the resulting hidden state vector $\textbf{y}_t$ goes through the output layer function in (\ref{lstm_7}), which is parametrized by its corresponding weights in $\Tt$, and generates the desired data $\dt$. We note that by the given $\htTt$, the data generating process in (\ref{lstm_ss})-(\ref{lstm_ss1}) is fully characterized by only the network weights $\Tt$ as the parameter-based EKF approach requires.}
 \label{fig:lstm_it}
\end{figure*}
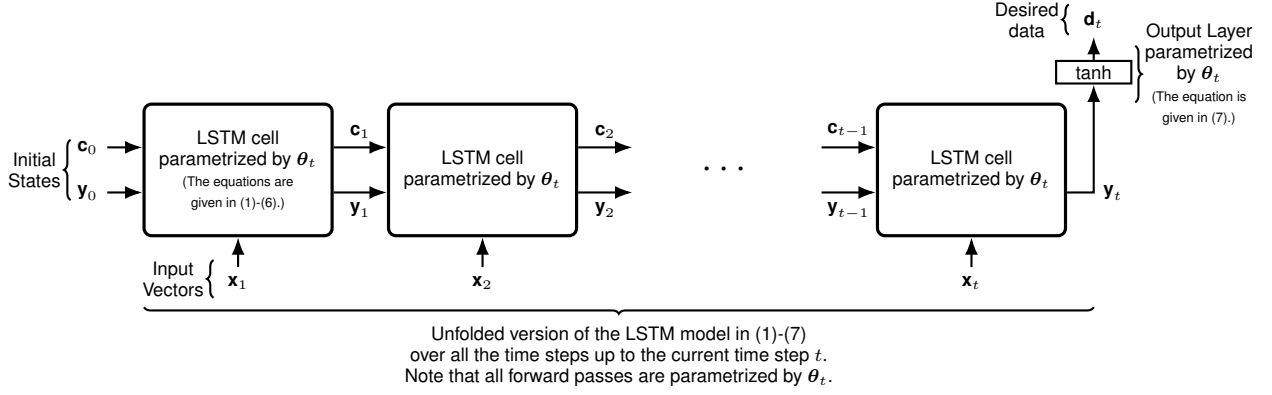

In order to convert the LSTM training into a state estimation problem, we model the desired signal as a nonlinear process realized by the LSTM network in (\ref{lstm_1})-(\ref{lstm_7}).  We note that since we use the parameter-based EKF approach, our desired signal model should be fully characterized by only the network weights $\bm{\theta}$. Therefore, we describe the underlying process of the incoming data with the following dynamical system:
\begin{align}
\Tt &= \bm{\theta}_{t-1}  \label{lstm_ss}\\
\dt&= \htTt. \label{lstm_ss1}
\end{align}
Here, we represent the LSTM weights that realize the incoming data stream with a vector $\Tt \in \mathbbm{R}^{n_{\theta}}$, which is modeled as a stationary process. As detailed in Fig. \ref{fig:lstm_it}, we use $\htTt$ to represent the unfolded version of the LSTM model in (\ref{lstm_1})-(\ref{lstm_7}) over all the time steps up to the current time step $t$, where all forward passes are parametrized by $\Tt$. Note that the dependence of $h_t(\cdot)$ on $t$ is due to the increased length of the recursion at each time step. The EKF learning algorithm is the EKF applied to the state-space model in (\ref{lstm_ss})-(\ref{lstm_ss1}) to estimate the network parameters $\bm{\theta}_t$. From the optimization perspective, EKF performs an online optimization procedure with the squared loss, aiming to predict the time-series with performance close to the best state-space model formulated as (\ref{lstm_ss})-(\ref{lstm_ss1}), i.e., the LSTM model with the (locally) optimum parameters~\cite{Tolga18}.

In the algorithm, we first perform  the forward LSTM-propagation with (\ref{lstm_1})-(\ref{lstm_7}) by using the parameters $\bm{\hat{\theta}}_t \in \mathbbm{R}^{n_{\theta}}$, which is our estimate for the optimum weights at time step $t$. Then, we perform the weight updates with the following formulas:
\begin{align}
&\bm{\hat{\theta}}_{t+1}= \bm{\hat{\theta}}_t + \textbf{G}_t ( \dt - \dht ) \label{sekf:1} \\
&\textbf{P}_{t+1} = (\textbf{I}-\textbf{G}_t\textbf{H}_t)\textbf{P}_t + \textbf{Q}_t  \label{cov_upd} \\
&\textbf{H}_t= \frac{\partial  \htT }{\partial \bm{\theta}}\Bigr|_{\bm{\theta}=\bm{\hat{\theta}}_t} \label{jacob} \\
&\textbf{G}_t=\textbf{P}_t\textbf{H}_t^T (\textbf{H}_t \textbf{P}_t \textbf{H}_t^T + \textbf{R}_t)^{-1}. \label{sekf:f} 
\end{align}
Here, $\textbf{P}_t \in \mathbbm{R}^{n_{\theta}\times n_{\theta}}$ is the state covariance matrix, which models the interactions between each pair of the LSTM parameters, $\textbf{G}_t \in \mathbbm{R}^{n_{\theta} \times n_d}$ is the Kalman gain matrix, and $\textbf{H}_t \in \mathbbm{R}^{n_d \times n_{\theta}}$ is the Jacobian matrix of $\htT$ evaluated at $\bm{\hat{\theta}}_t$. The noise covariance matrices $\Qt  \in \mathbbm{R}^{n_\theta \times n_\theta}$ and $\Rt \in \mathbbm{R}^{n_d \times n_d}$ are artifically introduced to the algorithm to enhance the training performance~\cite{Haykin}. In order to efficiently implement the algorithm, we use diagonal matrices for the artificial noise terms, i.e., $\Qt = q_t \I$ and $\Rt= r_t \I$, where $q_t, r_t > 0$. Due to (\ref{cov_upd}) and (\ref{sekf:f}), the computational complexity of the EKF learning algorithm is $O(n^2_{\theta})$, which is usually  prohibitive for the online settings.\footnote{We use big-$O$ notation, i.e., $O(f(x))$, to ignore constant factors.} 

\subsection{Adaptive Learning with IEKF}
In this study, we use IEKF to develop an efficient EKF-based training algorithm for LSTM-based adaptive learning. Recall that in IEKF, we approximate to the state covariance matrix by using a block-diagonal matrix approximation. To this end, here, we assume each neural node in LSTM as an independent subsystem, and use a different (and independent) EKF learning algorithm to learn the weight in each node. Let us denote the LSTM nodes with the first $(4 n_s + n_d)$ integers, i.e., $[(4n_s+ n_d)]=\{ 1 , \cdots, (4 n_s+ n_d) \}$, and use $i$ to index the nodes, i.e., $i \in [(4n_s+ n_d)]$. Then, we perform the weight updates in IEKF with the following:
\begin{align}
&\textrm{for } i= 1, \cdots, (4 n_s+ n_d) \nonumber \\
&\bm{\hat{\theta}}_{i,t+1}= \bm{\hat{\theta}}_{i,t} + \textbf{G}_{i,t} ( \textbf{d}_t - \bm{\hat{d}}_t  )  \label{iekf_1}\\
&\textbf{P}_{i,t+1} = (\textbf{I}-\textbf{G}_{i,t} \textbf{H}_{i,t})\textbf{P}_{i,t} + q_t \I \label{ind:cov_upd} \\
&\textbf{H}_{i,t}= \frac{\partial  \htT }{\partial \bm{\theta}_i}\Bigr|_{\bm{\theta}_i=\bm{\hat{\theta}}_{i,t}} \label{iekf_l} \\
&\textbf{G}_{i,t}=\textbf{P}_{i,t} \textbf{H}_{i,t}^T (\textbf{H}_{i,t} \textbf{P}_{i,t} \textbf{H}_{i,t}^T + r_{i,t} \I)^{-1}. \label{k_gain}
\end{align}
Here, $\Tit \in \mathbbm{R}^{\frac{n_\theta}{(4 n_s+ n_d)}}$ and $\Thit \in \mathbbm{R}^{\frac{n_\theta}{(4 n_s+ n_d)}}$ denote the optimal weights and their estimate in the LSTM node with index $i$. $\textbf{P}_{i,t} \in \mathbbm{R}^{\frac{n_\theta}{(4 n_s+ n_d)} \times \frac{n_\theta}{(4 n_s+ n_d)}}$ is the state covariance matrix, $\textbf{H}_{i,t} \in \mathbbm{R}^{n_d \times \frac{n_\theta}{(4 n_s+ n_d)}}$ is the Jacobian matrix of $\htTt$ and $\textbf{G}_{i,t} \in \mathbbm{R}^{\frac{n_\theta}{(4 n_s+ n_d)} \times n_d}$ is the Kalman gain matrix corresponding to the LSTM weights in node $i$. Since we perform (\ref{ind:cov_upd}) and  (\ref{k_gain}) $(4n_s+n_d)$ times, the computational complexity of the IEKF learning algorithm is $O(n^2_\theta/ (4 n_s+ n_d))$. We note that since LSTMs include a large number of nodes in practice, the reduction in the computational requirement with IEKF leads to considerable run-time savings in LSTM-based adaptive learning. However, as noted earlier,  since IEKF ignores the correlation between most of the network weights, it generally performs poorly compared to EKF~\cite{Haykin,Pusk94}.

In this study, we aim to introduce an efficient and effective second-order training algorithm by devising an algorithm that reduces the performance difference between EKF and IEKF.  To this end, differing from the existing studies~\cite{Pusk94,Shah92}, we introduce an IEKF-based LSTM training algorithm with a theoretical performance guarantee. Our simulations suggest that our theoretical guarantee provides a significant reduction in the performance difference between EKF and IEKF, hence, offers a desirable trade-off between speed and accuracy for second-order LSTM training. In the following sections, we present the development of our algorithm and show its performance gains in comparison to the state-of-the-art algorithms.

\section{Algorithm Development} \label{sec:main}
In this section, we introduce an adaptive IEKF-based training algorithm with a performance guarantee. We present our algorithm in two subsections: In the first subsection, we develop our algorithm by assuming that we have prior information about the data (in the form of the non-linear term in the error dynamics) before the training. In the second subsection, we drop this assumption and extend our algorithm to a truly online form, where the final algorithm sequentially learns the non-linear term without a priori knowledge of the data while preserving its performance guarantee.

For the analysis in the following subsections, we write the error dynamics of the independent EKF structures. To this end, we first write the Taylor series expansion of $\htTt$ around $\Tht$:
\begin{equation}
\label{eq:taylor1}
\htTt = \htTht + \Ht (\Tt-\Tht) +   \Xt,
\end{equation} 
where $\Ht \in \mathbbm{R}^{n_d \times n_\theta}$ is the Jacobian matrix of $\htT$ evaluated at $\Tht$, and $\Xt \in \mathbbm{R}^{n_d}$ is the non-linear term in the expansion. Note that $\dt = \htTt$ and $\dht = \htTht$.  For notational simplicity, we introduce two shorthand notations: $\zit= (\Tit-\Thit)$, and $\et= (\dt - \dht)$. Then, the error dynamics of the EKF learning algorithm applied to node $i$ can be written as:
\begin{align}
\et &= \sum_{i=1}^{(4 n_s + n_d)} \Hit \zit + \Xt  =  \Hit \zit + \Xit, \label{eq:taylor2}
\end{align}
where we consider the effect of partitioning the weights  as additional non-linearity for node $i$, i.e.,
\begin{equation}
\Xit=  \sum_{\substack{j=1 \\ j \neq i}}^{(4 n_s + n_d)} \Hjt \zjt + \Xt.
\end{equation}
For now, let us assume that the norm of $\Xit$ is bounded by a scalar value $\Xo$ for all the nodes throughout the training, i.e.,
\begin{equation}
\label{eq:xo}
\Xo \geq \norm{\Xit} \quad  \textrm{for all } t \in [T]\textrm{ and } i \in [(4 n_s + n_d)]. 
\end{equation} 
We note that we will prove (\ref{eq:xo}) in the following part (in Theorem \ref{th:bound}). 
 
\subsection{Performance Guarantee with Known Parameters}\label{sec:math}
\begin{algorithm}[t!]
\algsetup{linenosize=\normalsize}
	\caption{}\label{alg:alg1}
	\begin{algorithmic}[1]
	    \STATE \textbf{Parameters:} $\Xo \in \mathbbm{R}^+$.
		\STATE \textbf{Initialization:} $\textbf{P}_{i,1} = p_1 \I$ for $i \in [(4n_s+n_d)]$, where $p_1 \in \mathbbm{R}^+$.	  
		\FOR{$t=1$ \TO $T$}
		\STATE Generate $\dht$, observe $\dt$, and suffer a loss $\normet^2=\loss$.
		\IF{$\normet^2 > 4 \Xo^2$}
		\FOR{$i=1$ \TO $(4 n_s + n_d)$}
		\STATE Calculate the Jacobian matrix $\Hit$  \label{alg1:ht}
		\STATE $\displaystyle r_{i,t} = 3 \: \Tr( \Hit \Pit \Hit^T )/ n_d$  \label{alg1:rt}
		\STATE $\Kit=\Pit \Hit^T (\Hit \Pit \Hit^T + r_{i,t} \I)^{-1}$ \label{alg1:kt}
		\STATE $\Thitp=\Thit + \Kit (\dt-\dht)$ \label{alg1:up}
		\STATE $\Pitp= (\I - \Kit \Hit) \Pit + q_t \I$ \label{alg1:ptpp}
		\ENDFOR
		\ELSE
		\STATE $\Thitp=\Thit$, for $i \in [(4 n_s + n_d)]$ \label{alg1:unch1}
		\STATE $\Pitp = \Pit$, for $i \in [(4 n_s + n_d)]$  \label{alg1:unch2}
		\ENDIF
		\ENDFOR
	\end{algorithmic}
\end{algorithm}
In this subsection, we present an IEKF-based algorithm that guarantees errors to converge to a small interval under the assumption that $\Xo$ is known, i.e., Algorithm \ref{alg:alg1}. 

In Algorithm \ref{alg:alg1}, we take  the upper bound of the residual terms $\Xo$  as the input. We initialize the state covariance matrix of each independent EKF as $\textbf{P}_{i,1} = p_1 \I$, where $p_1 \in \mathbbm{R}^+$. In each time step, we first generate a prediction $\dht$, then receive the desired data $\dt$, and suffer a loss $\normet^2=\loss$.  We perform the parameter updates only if the loss is bigger than $4 \Xo^2$, i.e., $\norm{\et}^2 > 4 \Xo^2$. If so, we compute the Jacobian matrix $\Hit$, calculate measurement noise level  $r_{i,t}$ in lines \ref{alg1:ht} and \ref{alg1:rt}. In the following, we show that our $r_{i,t}$ selection guarantees the stability of our algorithm (see Theorem \ref{th:main}). We, then, calculate the Kalman gain matrix $\Kit$ for each $i \in [(4 n_s +n_d)]$, update the weights  and  state covariance matrix of the weights belonging to node $i$ in lines \ref{alg1:kt}-\ref{alg1:ptpp}.

In the following lemma,  we present several propositions that will be used to prove the theoretical guarantees of Algorithm  \ref{alg:alg1}.
\begin{lemma}
\label{lem:statements}
For $\{ t : \normet^2 > 4 \Xo^2 \}$,  Algorithm \ref{alg:alg1} guarantees the following statements:
\begin{enumerate}
\item For each node $i$, the difference between the locally optimal weights and LSTM weights is governed with the following equation:
\begin{equation}
\label{eq:ed1}
\zitp= ( \I - \Kit \Hit) \zit - \Kit \Xit,
\end{equation}
which can also be written as 
\begin{equation}
\label{eq:ed2}
\zitp- \zit = - \Kit \Hit \zit - \Kit \Xit.
\end{equation}
\item For each node $i$, $\Piti$ and $(\Pitp- q_t \I)^{-1}$ exist and they are always positive definite as such
\begin{align}
(\Pitp- q_t \I)^{-1} &= \Piti ( \I - \Kit \Hit)^{-1} \label{eq:pthi1} \\
&= \Piti + r^{-1}_{i,t} \Hit^T  \Hit \geq \textbf{0}. \label{eq:pthi2}
\end{align}
\item As a result of the previous two statements,
\begin{equation}
(\Pitp- q_t \I)^{-1}  \zitp = \Piti  \zit - (\Pitp- q_t \I)^{-1} \Kit \Xit \label{eq:pt_ed} 
\end{equation}
holds for each node $i \in [(4 n_s + n_d)]$.
\end{enumerate}
\end{lemma} 
\begin{proof}
See Appendix \ref{sec:appa}.
\end{proof}

In the following theorem, we state the theoretical guarantees of Algorithm \ref{alg:alg1}.
\begin{theorem}
\label{th:main}
If $\sum_{i=1}^{(4n_s+n_d)} \Tr( \Pit)$ stays bounded during training, Algorithm \ref{alg:alg1} guarantees the following statements:
\begin{enumerate}
\item The LSTM weights stay bounded during training.
\item The loss sequence $\{\norm{\et}^2\}_{t \geq 1}$ is guaranteed to converge to the interval $[0, 4 \Xo^2]$. 
\end{enumerate}
\end{theorem}
\begin{proof}
See Appendix \ref{sec:appa}.
\end{proof}

\begin{remark}
\label{rem:1}
Due to the Kalman gain matrix formulation (line \ref{alg1:kt} in Algorithm \ref{alg:alg1}), $\Tr( (\I - \Kit \Hit) \Pit )$ is always smaller than  or equal to $\Tr( \Pit)$ for each node $i$, i.e., $ \Tr( (\I - \Kit \Hit) \Pit ) \leq  \Tr( \Pit)$ for all $i \in [(4 n_s+n_d)]$. Since $\Pitp= (\I - \Kit \Hit) \Pit + q_t \I$, and the artificial process noise level $q_t$ is a user-dependent parameter, the condition in Theorem \ref{th:main} can be satisfied by the user by selecting sufficiently small $q_t$.
\end{remark}

We note that due to $\dt, \dht \in [-1,1]^{n_d}$, $\norm{\et}^2 \leq 4 n_d$. Therefore, to ensure that the 2nd statement in Theorem \ref{th:main} is not a trivial interval, we must show that $\Xo \in [0, \sqrt{n_d}]$. In the following theorem we show that choosing small initial weights, i.e., $\Thb \approx \textbf{0}$, guarantees $\Xo \in [0, \sqrt{n_d}]$.

\begin{theorem}
\label{th:bound}
For any bounded data sequence $\{\dt\}_{ t \geq 1}$ with $\dt \in  [-1,1]^{n_d}$, there exists a small positive number $\epsilon$ such that $\norm{\Thb} \leq \epsilon$ ensures that $\Xo$ is in the interval $[0, \sqrt{n_d}]$.
\end{theorem}
\begin{proof}
See Appendix \ref{sec:appa}.
\end{proof}
\begin{remark}
We note that although Theorem \ref{th:bound} guarantees the existence of $\epsilon$, which guarantees $\Xo \in [0, \sqrt{n_d}]$ under the condition of $\norm{\Thb} \leq \epsilon$, it does not provide us a specific $\epsilon$ value. However, in the simulations, we observe that $\Thb \sim \mathcal{N}(\textbf{0},0.01 \textbf{I})$ gives us both small error rates and fast convergence speed at the same time. Therefore, in the following we assume that $\Thb \sim \mathcal{N}(\textbf{0},0.01 \textbf{I})$ is sufficiently small to ensure $\Xo \in [0,\sqrt{n_d}]$ for practical applications.
\end{remark}

Note that by Theorem \ref{th:bound}, we guarantee an interval for $\Xo$. In the following section, we utilize this interval to extend our algorithm to a truly online form.

\subsection{Truly Online Form}\label{sec:math}
\begin{algorithm}[t!]
\algsetup{linenosize=\normalsize}
	\caption{}\label{alg:alg2}
	\begin{algorithmic}[1]
	    \STATE \textbf{Parameters:} $\Xo_\textrm{min} \in \mathbbm{R}^+$.
		\STATE Initialize $\bm{\Xo} = [  \sqrt{n_d}, 0.5 \sqrt{n_d},0.25 \sqrt{n_d}, \cdots, \Xo_\textrm{min}]^T $. \label{alg2:init}
		\STATE Set $\displaystyle N=  \log(\sqrt{n_d} / \Xo_\textrm{min}) +1$. \label{alg2:N}
		\STATE Initialize  $N$ independent Algorithm \ref{alg:alg1} instances with the entries of $\bm{\Xo}$. \label{alg2:init2}
		\STATE Let the indices of the instances be $j \in [N]$.
		\STATE Initialize the weight of the instances as $w_{j,1} = 1/N$, for $j \in [N]$.
		\FOR{$t=1$ \TO $T$}
		\STATE Receive $\{\dhjt\}_{j \in [N]}$ vectors of the Algorithm \ref{alg:alg1} instances. \label{alg2:rec}
		\STATE Calculate the prediction vector as $\dht = \frac{\sum_{j=1}^N w_{j,t} \dhjt}{\sum_{k=1}^N w_{k,t}}$. \label{alg2:calc}
		\STATE Observe $\dt$ and suffer a loss $\normet^2= \norm{\dt- \dht}^2$. \label{alg2:obs}
		\STATE Update all the  Algorithm \ref{alg:alg1} instances by using $\dt$ and their own $\dhjt$ vector. \label{alg2:upd1}
		\STATE $w_{j,t+1} = w_{j,t} \exp(- \frac{1}{8 n_d} \norm{\dt - \dhjt}^2)$, for $j \in [N]$. \label{alg2:upd2}
		\ENDFOR
	\end{algorithmic}
\end{algorithm}

In this section, we extend Algorithm \ref{alg:alg1} to a truly online form, where we do not necessarily know $\Xo$ a priori. To this end, we introduce Algorithm \ref{alg:alg2}, where we run multiple instances of Algorithm \ref{alg:alg1}  with carefully selected $\Xo$ values, and mix their predictions with the exponential weighting algorithm to find the smallest $\Xo$ efficiently in a truly online manner.

In Algorithm \ref{alg:alg2}, we use the fact that the effective range of $\Xo$ is $[0, \sqrt{n_d}]$. Here, we specify a small $\Xo_\textrm{min}$ and run multiple Algorithm \ref{alg:alg1} instances independently with $\Xo$ values from $\sqrt{n_d}$ to $\Xo_\textrm{min}$ decreasing with powers of $2$, i.e., $\bm{\Xo} = [ \sqrt{n_d}, 0.5 \sqrt{n_d},0.25 \sqrt{n_d}, \cdots, \Xo_\textrm{min}]^T $, where we have a total of $N=\log( \sqrt{n_d} / \Xo_\textrm{min}) +1$ number of independent Algorithm \ref{alg:alg1} instances. In each round, we receive the prediction of the instances, i.e., $\{\dhjt\}_{j \in [N]}$, and take the weighted average of $\{\dhjt\}_{j \in [N]}$ to determine $\dht$, i.e., $\dht = (\sum_{j=1}^N w_{j,t} \dhjt)/\sum_{k=1}^N w_{k,t}$. In its following, we observe the target value $\dt$, and suffer a loss $\normet^2 = \norm{\dt - \dht}^2$. We, then, update the Algorithm \ref{alg:alg1} instances with their own predictions $\dhjt$, and update the weights as $w_{j,t+1} = w_{j,t} \exp(- \frac{1}{8 n_d} \norm{\dt - \dhjt}^2)$ for $j \in [N]$.

In the following theorem, we derive the theoretical guarantees of Algorithm \ref{alg:alg2}.
\begin{theorem}
\label{th:main2}
Let us use $\Xo_\textrm{best}$ to denote the smallest possible value for $\Xo$ that guarantees the tightest possible interval for $\{\norm{\et}^2\}_{t \geq 1}$. Assuming $\Xo_\textrm{best} \in [ \Xo_\textrm{min}, \sqrt{n_d}]$, Algorithm \ref{alg:alg2} guarantees that $\{\norm{\et}^2\}_{t \geq 1}$ converges to the interval $[0, 16 \Xo_\textrm{best}^2]$ in a truly online manner. 
\end{theorem}
\begin{proof}
See Appendix \ref{sec:appb}.
\end{proof}

\begin{remark}
We note that the results so far do not assume any specific distribution for the data nor any topology for the error surface. Therefore, our performance guarantee holds globally regardless of the distribution of the data, given that the target sequence is bounded (required in Theorem \ref{th:bound}). The difference between different noise levels or two local optima demonstrates itself in $\Xo$, or in the asymptotic error interval in Theorem \ref{th:main}, which is guaranteed to be learned in a data-dependent manner in Theorem \ref{th:main2}.
\end{remark}

Now that we have proved the performance guarantee of our algorithm, in the following remark, we present the computational complexity of Algorithm \ref{alg:alg2}, and compare the presented complexity with the computational requirement of EKF.

\begin{remark}
We maintain that $\Xo_\textrm{min} = 0.01$ is practically sufficient for $\Xo_\textrm{best} \in [ \Xo_\textrm{min}, \sqrt{n_d}]$ (or to guarantee a tight interval for error to converge). In this case, the number of independent Algorithm \ref{alg:alg1} instances in Algorithm \ref{alg:alg2}, i.e., $N$, becomes $\log (100 \sqrt{n_d})$, equivalently, $ (7+ 0.5 \log n_d)$. Hence, the computational complexity of Algorithm \ref{alg:alg2} becomes $O\big((7+ 0.5 \log n_d) n_\theta^2/ (4 n_s+n_d)\big)$ in the worst case, where we assume that all the Algorithm \ref{alg:alg1} instances perform the IEKF updates (lines  \ref{alg1:ht}-\ref{alg1:ptpp} in Algorithm \ref{alg:alg1}) in every time step.

As noted earlier, the computational requirement of EKF is $O(n_\theta^2)$. Therefore, the asymptotic efficiency gain of Algorithm \ref{alg:alg2} over EKF is $ (4 n_s+n_d)/(7+ 0.5 \log n_d )$. Since the number of nodes in practical LSTMs is usually more than $50$, in the worst case, Algorithm \ref{alg:alg2} reduces the computational requirement of the EKF-based training by $5-7$ times~\cite{lstmdekf, Coskun17, Tolga18, Yang17}. However, we note that in practice, Algorithm \ref{alg:alg1} with a high $\Xo \: (\geq 0.2)$ performs a small number of updates, usually around $20-30$ updates in $1000$ time steps. Therefore, the ratio between the run-times of EKF and Algorithm \ref{alg:alg2} is generally considerably higher than the worst-case ratio derived above. In fact, in the following section, we demonstrate that Algorithm \ref{alg:alg2} provides a $10-15$ times reduction in the training time compared to EKF while yielding very similar performance.
\end{remark}

\section{Simulations}
\label{sec:exp}


\begin{table*}[t]

\begin{subtable}{\linewidth}
\centering
\scalebox{0.92}{
\begin{tabular}{@{}l|c|c|c|c|c|c|c|c|c|@{}}
\cmidrule(l){2-10}                   
& \multicolumn{3}{c|}{kin8nm ($n_h = 16$, $n_x = 9$)}                                
& \multicolumn{3}{c|}{kin32fm ($n_h = 12$, $n_x = 33$)}                               
& \multicolumn{3}{c|}{elevators ($n_h = 12$, $n_x = 19$)}                             \\ \cmidrule(l){2-10} 
                              & NSE                    & kNSE                   & Run-time 
                              & NSE                    & kNSE                   & Run-time 
                              & NSE                    & kNSE                   & Run-time \\ \midrule
\multicolumn{1}{|l|}{SGD}     & $0.70\pm 0.23$          & $0.67\pm 0.24$          & $0.58$     
                              & $0.42\pm 0.27$          & $0.42\pm 0.29$          & $1.15$     
                              & $0.51\pm 0.21$          & $0.52\pm 0.24$          & $0.62$     \\ \midrule
\multicolumn{1}{|l|}{RMSprop} & $0.44\pm 0.26$          & $0.44\pm 0.25$          & $0.59$     
                              & $0.24\pm 0.18$          & $0.24\pm 0.18$          & $1.03$     
                              & $0.34\pm 0.17$          & $0.35\pm 0.17$          & $0.55$     \\ \midrule
\multicolumn{1}{|l|}{Adam}    & $0.42\pm 0.23$          & $0.43\pm 0.23$          & $0.58$     
							& $0.24\pm 0.19$          & $0.25\pm 0.19$          & $1.06$     
							& $0.34\pm 0.16$          & $0.35\pm 0.17$          & $0.57$     \\ \midrule
\multicolumn{1}{|l|}{DEKF}    & $0.39\pm 0.23$          & $0.40\pm 0.23$          & $1.68$     
                              & $0.23\pm 0.16$          & $0.23\pm 0.17$          & $2.51$     
                              & $0.24\pm 0.14$          & $0.24\pm 0.14$          & $1.74$     \\ \midrule
\multicolumn{1}{|l|}{EKF}     & $\bm{0.31\pm 0.17}$     & $\bm{0.30\pm 0.18}$     & $62.17$    
                              & $\bm{0.18\pm 0.11}$     & $\bm{0.18\pm 0.12}$     & $112.11$   
                              & $\bm{0.19\pm 0.09}$     & $\bm{0.19\pm 0.09}$     & $53.74$    \\ \midrule
\multicolumn{1}{|l|}{Alg2 (This work)}    & $\bm{0.32\pm 0.20}$     & $\bm{0.34\pm 0.20}$     & $4.48$     
                                          & $\bm{0.15\pm 0.06}$     & $\bm{0.18\pm 0.10}$     & $9.04$    
                                          & $\bm{0.21\pm 0.09}$     & $\bm{0.22\pm 0.09}$     & $4.45$     \\ \bottomrule
\end{tabular} }
\caption{}\label{tab:1a}
\end{subtable}

\vspace{1mm}
\begin{subtable}{\linewidth}
\centering
\scalebox{0.92}{
\begin{tabular}{@{}l|c|c|c|c|c|c|c|c|c|@{}}
\cmidrule(l){2-10}
& \multicolumn{3}{c|}{puma8nm ($n_h = 16$, $n_x = 9$)}                               
& \multicolumn{3}{c|}{puma8nh ($n_h = 16$, $n_x = 9$)}                               
& \multicolumn{3}{c|}{puma32fm ($n_h = 12$, $n_x = 33$)}                              \\ \cmidrule(l){2-10} 
                              & NSE                    & kNSE                   & Run-time 
                              & NSE                    & kNSE                   & Run-time 
                              & NSE                    & kNSE                   & Run-time \\ \midrule
\multicolumn{1}{|l|}{SGD}     & $0.46\pm 0.22$          & $0.45\pm 0.22$          & $0.58$     
                              & $0.72\pm 0.21$          & $0.72\pm 0.22$          & $0.61$     
                              & $0.34\pm 0.20$          & $0.33\pm 0.22$          & $1.11$     \\ \midrule
\multicolumn{1}{|l|}{RMSprop} & $0.25\pm 0.13$          & $0.26\pm 0.13$          & $0.59$     
						    & $0.52\pm 0.16$          & $0.52\pm 0.16$          & $0.61$     
						    & $0.20\pm 0.13$          & $0.20\pm 0.14$          & $1.03$     \\ \midrule
\multicolumn{1}{|l|}{Adam}    & $0.26\pm 0.16$          & $0.26\pm 0.16$          & $0.58$     
                              & $0.52\pm 0.18$          & $0.52\pm 0.18$          & $0.65$     
                              & $0.22\pm 0.17$          & $0.22\pm 0.17$          & $1.06$     \\ \midrule
\multicolumn{1}{|l|}{DEKF}    & $0.21\pm 0.13$          & $0.21\pm 0.13$          & $1.72$     
                              & $0.50\pm 0.15$          & $0.50\pm 0.15$          & $1.76$     
                              & $0.19\pm 0.13$          & $0.18\pm 0.13$          & $2.51$     \\ \midrule
\multicolumn{1}{|l|}{EKF}     & $\bm{0.14\pm 0.08}$     & $\bm{0.14\pm 0.08}$     & $59.38$    
                              & $\bm{0.47\pm 0.17}$     & $\bm{0.46\pm 0.16}$     & $68.23$    
                              & $\bm{0.15\pm 0.09}$     & $\bm{0.15\pm 0.09}$     & $113.14$   \\ \midrule
\multicolumn{1}{|l|}{Alg2 (This work)}    & $\bm{0.16\pm 0.11}$     & $\bm{0.17\pm 0.11}$     & $4.31$     
                                          & $\bm{0.47\pm 0.21}$     & $\bm{0.46\pm 0.21}$     & $4.60$     
                                          & $\bm{0.12\pm 0.05}$     & $\bm{0.13\pm 0.08}$     & $9.21$     \\ \bottomrule
\end{tabular} }
\caption{}\label{tab:1b}
\end{subtable}
\caption{Empirical $90\%$ confidence intervals for the one-step and $k$-step normalized squared errors, i.e., NSE and kNSE, with the run-times (in seconds) of the compared algorithms. The two algorithms with the lowest median NSEs are emphasized. The simulations are performed on a computer with i7-7500U processor, 2.7-GHz CPU, and 8-GB RAM. } 
\label{tab:1}
\end{table*}

\begin{figure*}[t]
\begin{subfigure}[t]{0.5\textwidth}
    \centering
	\includegraphics[width=0.95\textwidth]{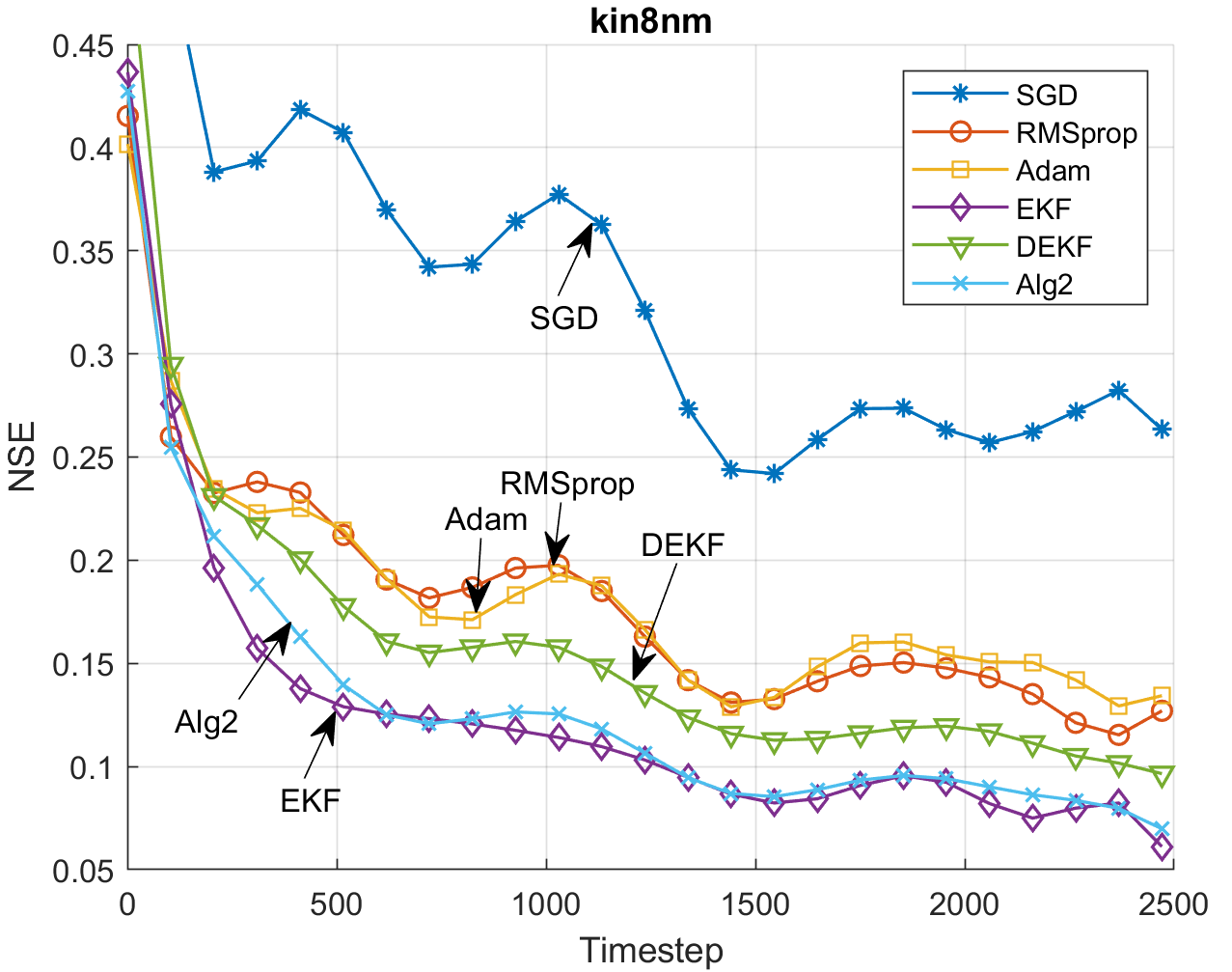}
	\caption{} \label{fig1:r1}
\end{subfigure}
\begin{subfigure}[t]{0.5\textwidth}
    \centering
	\includegraphics[width=0.95\textwidth]{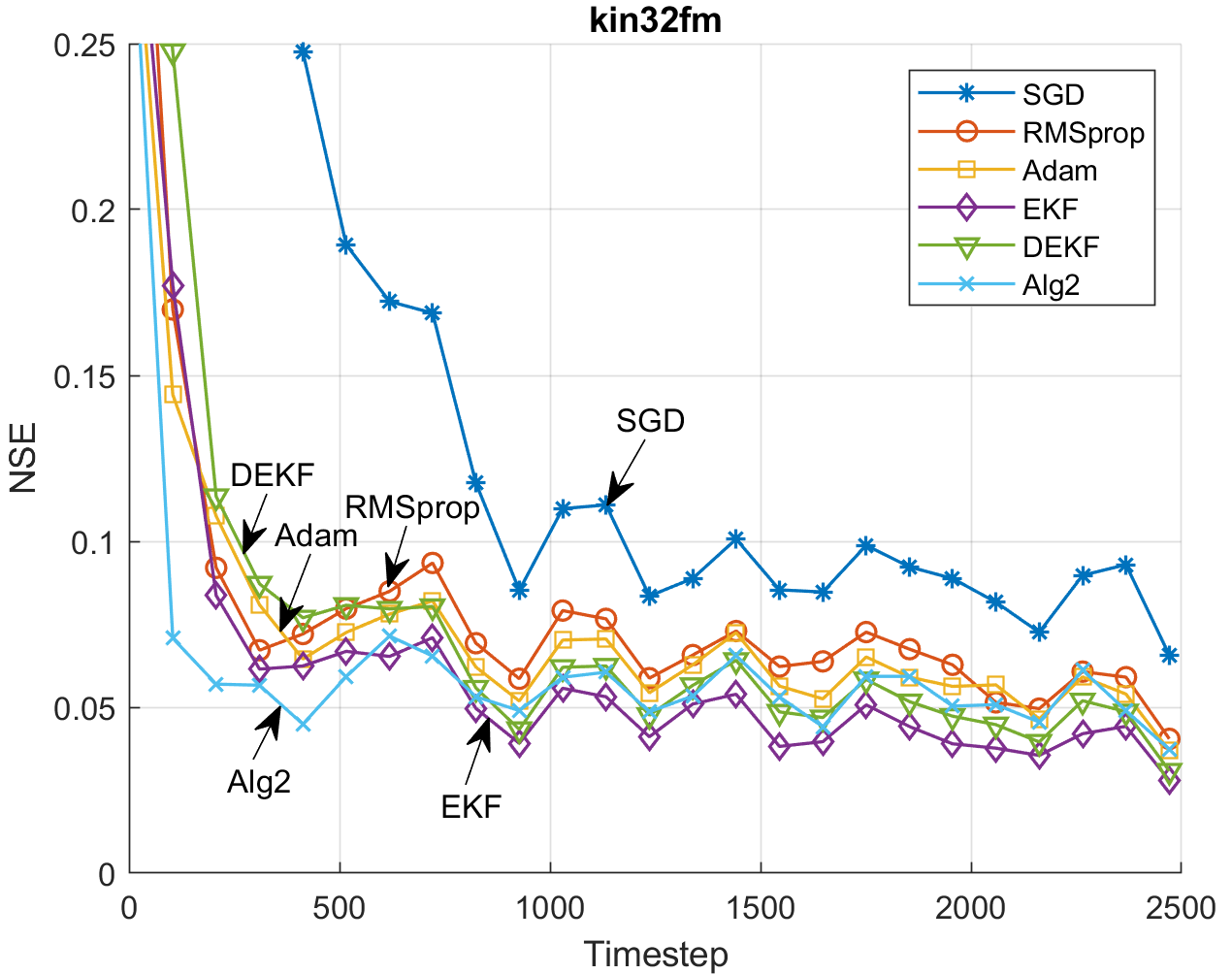}
	\caption{} \label{fig1:r2}
\end{subfigure}  
\caption{Sequential prediction performances of the algorithms on the (a) kin8nm, (b) kin32fm datasets.}\label{fig:1}
\end{figure*}

In this section, we illustrate the performance improvements of our algorithm, i.e., Algorithm \ref{alg:alg2}, with respect to the state-of-the-art. To that end, we compare Algorithm \ref{alg:alg2} (abbreviated as Alg2) with five widely used optimization algorithms, i.e., SGD, Adam, RMSprop, DEKF, and EKF, on six real-world and three synthetic datasets. 

In all simulations, we randomly draw the initial weights from a Gaussian distribution with zero mean and standard deviation of $0.1$, i.e., $\Thb \sim \mathcal{N}(\textbf{0},0.01\I)$. We set the initial values of all state variables to $0$ and $\Xo_\textrm{min}$ to $0.01$. To satisfy the condition in Theorem \ref{th:main2}, we linearly map the target vectors between $[-1,1]$. 

We report results with $90\%$ confidence intervals. To obtain the confidence intervals, we repeat the experiments 20 times with randomly chosen initial parameters (generated with fixed and different seeds), and calculate the $5$th and $95$th percentiles of the errors in each time step. Then, we report the average of the calculated intervals between the calculated percentiles.

We search the hyperparameters of the learning algorithms (except EKF) over a dense grid. Since the grid search for EKF takes a very long time, we re-use the same hyperparameters specified for DEKF. Due to its large overall run-time, we only consider the first $1000$ input/output pairs of the datasets in the grid-search. For each hyperparameter in each optimization method, we repeat the learning procedure ten
times with randomly chosen initial parameters (generated with fixed and different seeds), and report the results using the hyperparameters that minimize the mean of the mean squared errors obtained in the ten runs.  

To provide a consistent scale for the results, we normalize the squared errors with the variance of the target stream. To evaluate the generalization performance of the algorithms, we report the confidence intervals of both one-step normalized squared errors (shortly NSE) and $k$-step normalized squared errors (shortly kNSE), where kNSE is  the normalized squared error of the $k$-step ahead forecast.  In the experiments, we use $k = 50$. To visualize the learning performances over time, we plot the learning curves, which illustrates the median NSE obtained by the algorithms across time. We share the source code of our experiments on GitHub at \url{https://github.com/nurimertvural/EfficientEffectiveLSTM}.

\subsection{Real Data Benchmark}

\begin{figure*}[t]
\begin{subfigure}[t]{0.5\textwidth}
    \centering
	\includegraphics[width=0.95\textwidth]{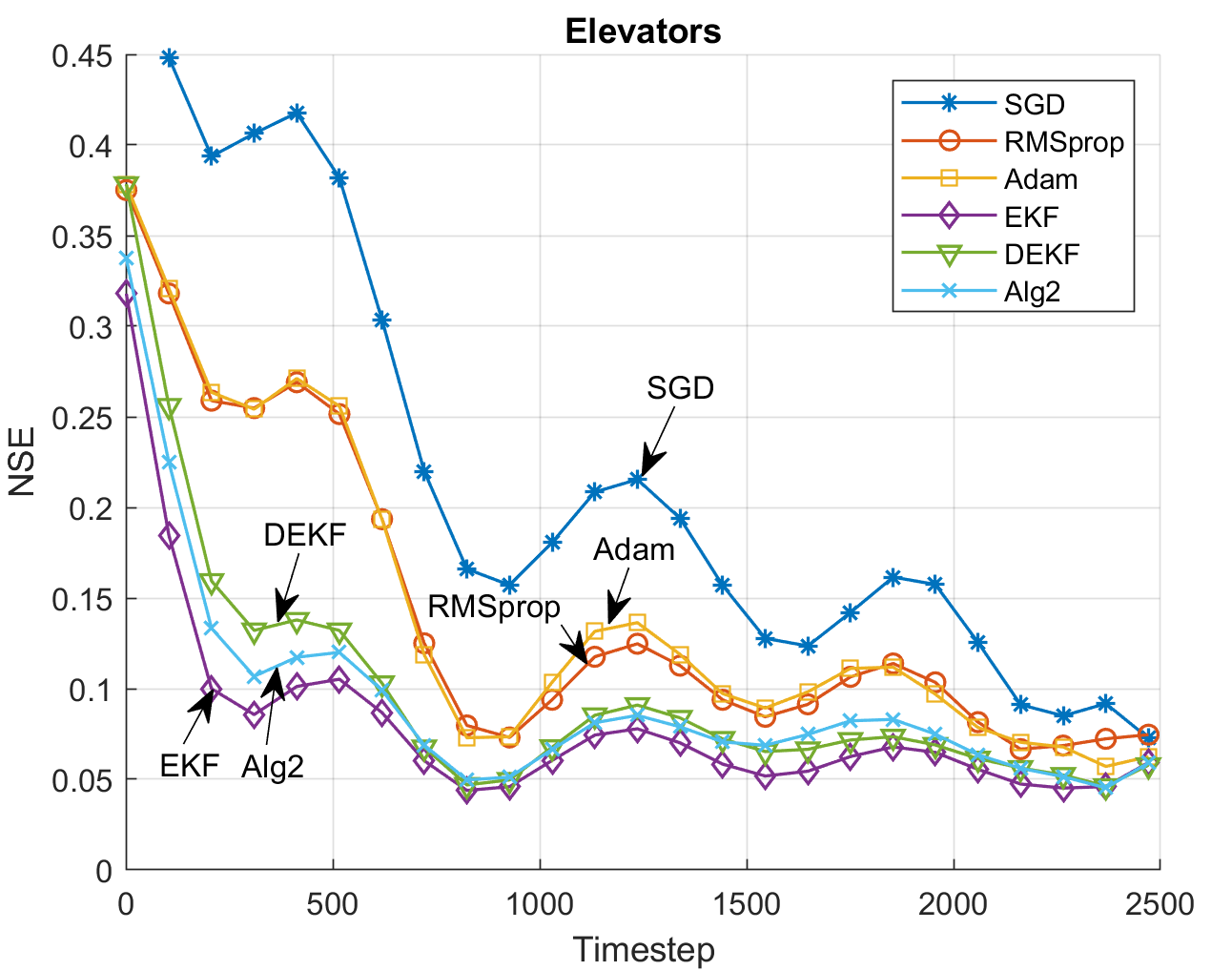}
	\caption{} \label{fig2:r1}
\end{subfigure}
\begin{subfigure}[t]{0.5\textwidth}
    \centering
	\includegraphics[width=0.95\textwidth]{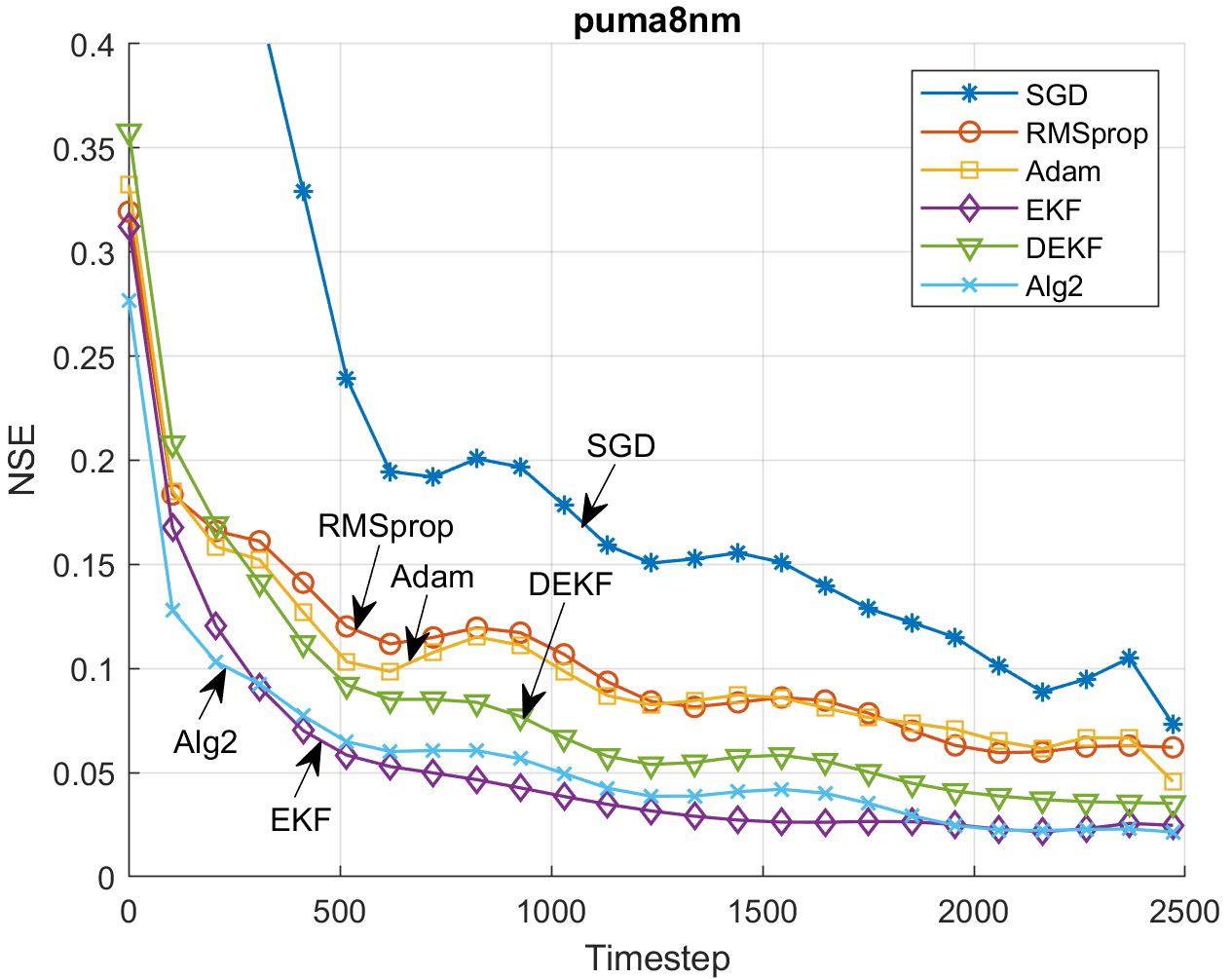}
	\caption{} \label{fig2:r2}
\end{subfigure}
\begin{subfigure}[t]{0.5\textwidth}
    \centering
	\includegraphics[width=0.95\textwidth]{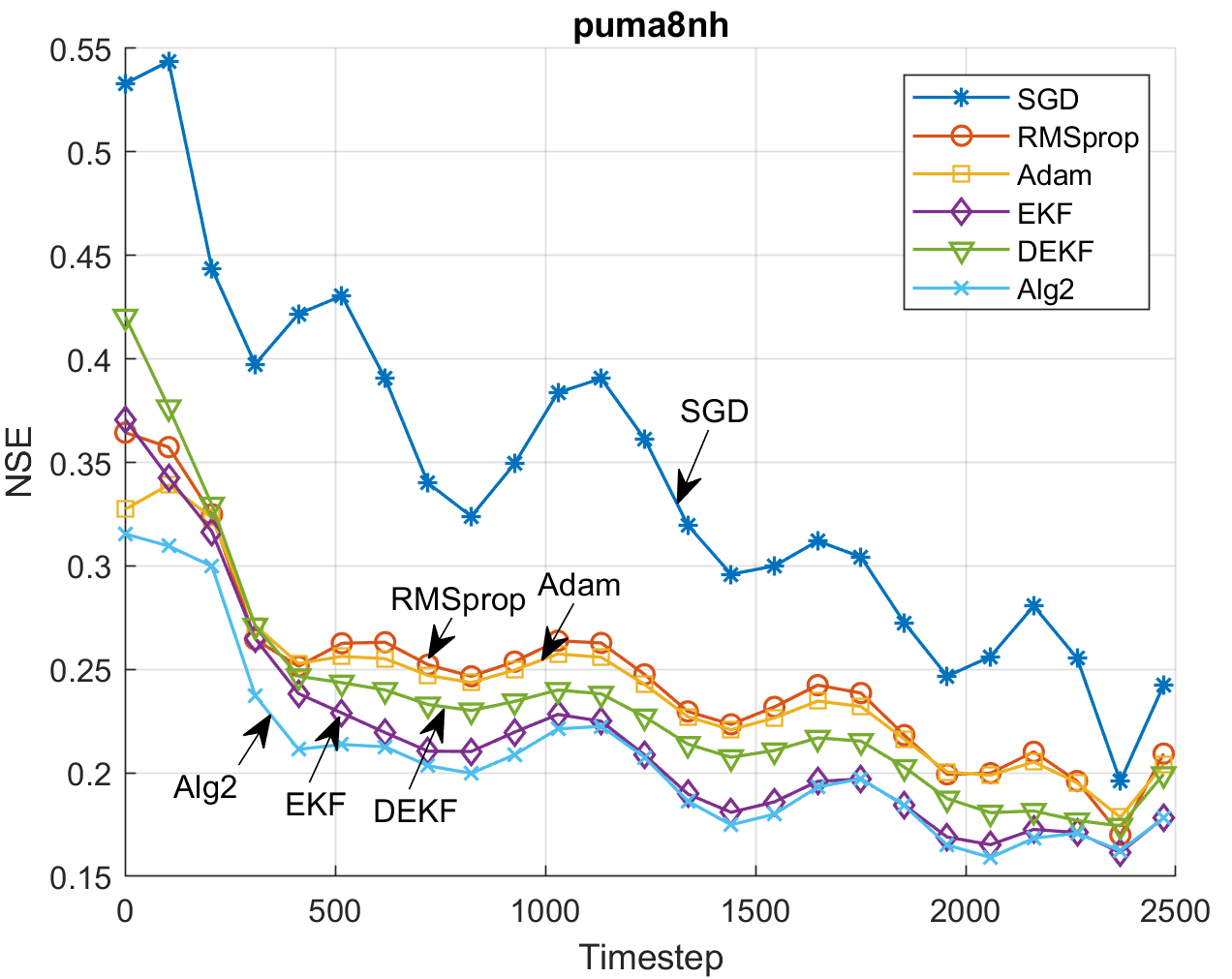}
	\caption{} \label{fig2:r3}
\end{subfigure}
\begin{subfigure}[t]{0.5\textwidth}
    \centering
	\includegraphics[width=0.95\textwidth]{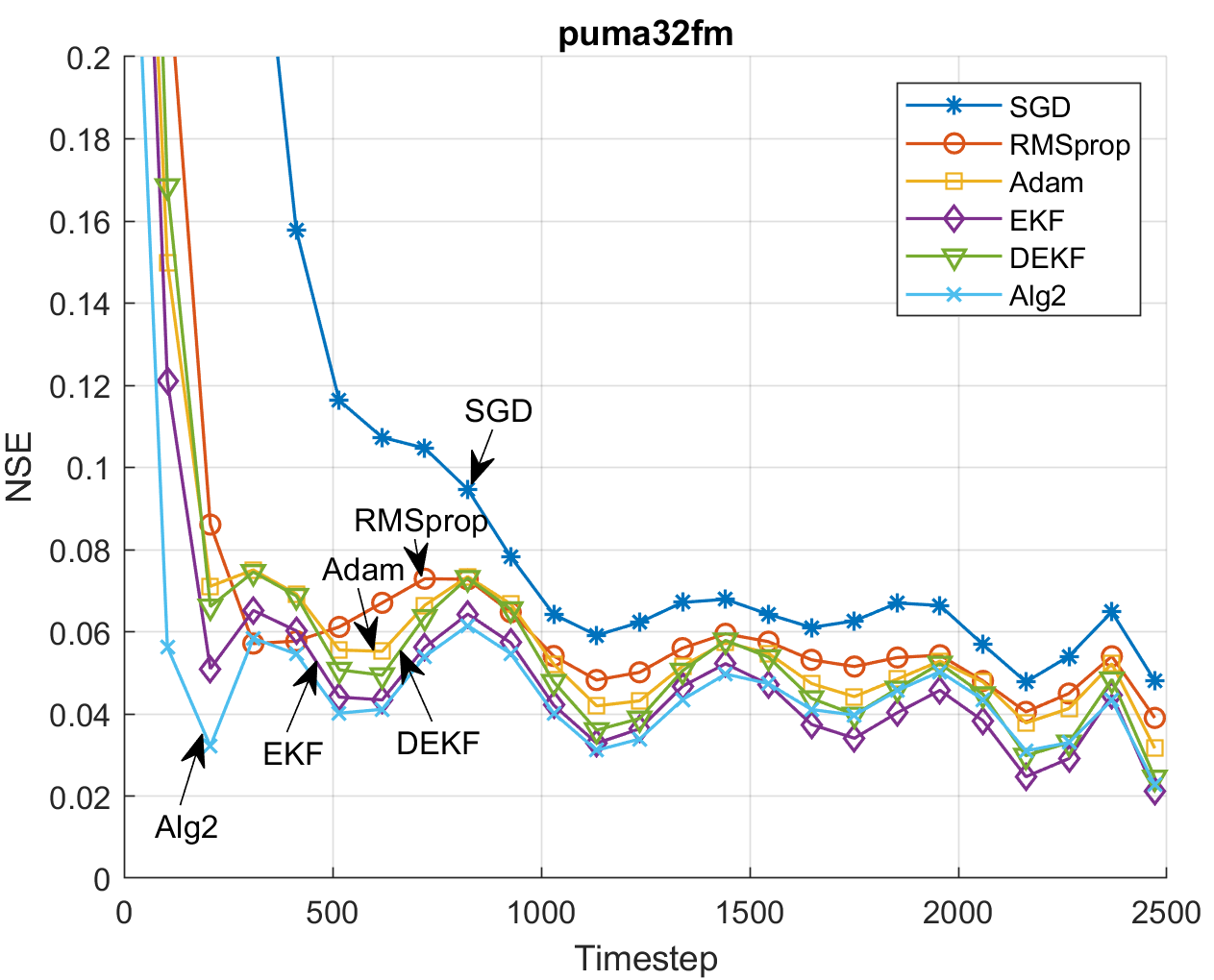}
    \caption{} \label{fig2:r4}
\end{subfigure}
\caption{Sequential prediction performances of the algorithms on the (a) elevators, (b) puma8nm, (c) puma8nh, and (d) puma32fm datasets.}\label{fig:2}
\end{figure*}

\begin{table*}[t]

\begin{subtable}{\linewidth}
\centering
\begin{tabular}{@{}|c|c|c|c|c|c|c|c|@{}}
\toprule
\multicolumn{2}{|c|}{RMSprop} & \multicolumn{2}{c|}{DEKF} 
& \multicolumn{2}{c|}{EKF}    & \multicolumn{2}{c|}{Alg2 (This work)} \\ \midrule
  Timestep      & Run-time    & Timestep      & Run-time      
& Timestep      & Run-time    & Timestep      & Run-time        \\ \midrule
$43778$         & $15.61$     & $29995$       & $11.24$         
& $\bm{5995}$   & $37.93$     & $6906$        & $\bm{11.21}$    \\ \midrule
$44472$         & $15.74$     & $11464$       & $\bm{5.21}$ 
& $\bm{2651}$   & $17.16$     & $6107$        & $10.26$         \\ \midrule
$22912$         & $\bm{8.47}$ & $25382$       & $9.72$         
& $\bm{5411}$   & $33.74$     & $7065$        & $11.79$         \\ \midrule
$27264$         & $7.46$      & $12242$       & $\bm{5.25}$ 
& $\bm{4205}$   & $26.5$      & $8245$        & $14.48$         \\ \midrule
$21073$         & $\bm{6.69}$ & $19491$       & $7.67$          
& $\bm{3244}$   & $20.65$     & $5510$        & $9.26$          \\ \bottomrule
\end{tabular}
\caption{Adding three binary sequences ($n_h = 12$, $n_x =4$).}\label{tab:2a}
\end{subtable}

\vspace{1mm}
\begin{subtable}{\linewidth}
\centering
\begin{tabular}{@{}|c|c|c|c|c|c|c|c|@{}}
\toprule
  \multicolumn{2}{|c|}{RMSprop} & \multicolumn{2}{c|}{DEKF}   
& \multicolumn{2}{c|}{EKF}      & \multicolumn{2}{c|}{Alg2 (This work)}       \\ \midrule
  Timestep      & Run-time      & Timestep     & Run-time     
& Timestep      & Run-time      & Timestep     & Run-time        \\ \midrule
\multicolumn{2}{|c|}{Failed}    & \multicolumn{2}{c|}{Failed} 
& $\bm{28978}$  & $197.03$      & $31353$      & $\bm{52.05}$    \\ \midrule
\multicolumn{2}{|c|}{Failed}    & \multicolumn{2}{c|}{Failed} 
& $\bm{30578}$  & $207.38$      & $36606$      & $\bm{61.41}$    \\ \midrule
\multicolumn{2}{|c|}{Failed}    & \multicolumn{2}{c|}{Failed} 
& $59835$       & $407.63$      & $\bm{53204}$ & $\bm{88.03}$    \\ \midrule
\multicolumn{2}{|c|}{Failed}    & \multicolumn{2}{c|}{Failed} 
& $50549$       & $343.84$      & $\bm{18083}$ & $\bm{30.99}$    \\ \midrule
\multicolumn{2}{|c|}{Failed}    & \multicolumn{2}{c|}{Failed} 
& $\bm{15871}$  & $107.75$      & $18165$      & $\bm{32.96}$    \\ \bottomrule
\end{tabular}
\caption{Adding four binary sequences ($n_h = 12$, $n_x =5$).}\label{tab:2b}
\end{subtable}

\vspace{1mm}
\begin{subtable}{\linewidth}
\centering
\begin{tabular}{@{}|c|c|c|c|c|c|c|c|@{}}
\toprule
\multicolumn{2}{|c|}{RMSprop}   & \multicolumn{2}{c|}{DEKF}   
& \multicolumn{2}{c|}{EKF}      & \multicolumn{2}{c|}{Alg2 (This work)}        \\ \midrule
  Timestep      & Run-time      & Timestep     & Run-time     
& Timestep      & Run-Time      & Timestep     & Run-Time        \\ \midrule
\multicolumn{2}{|c|}{Failed}    & \multicolumn{2}{c|}{Failed} 
& $\bm{27977}$ & $202.27$       & $33912$      & $\bm{52.34}$    \\ \midrule
\multicolumn{2}{|c|}{Failed}    & \multicolumn{2}{c|}{Failed} 
& $\bm{61709}$ & $477.21$       & $78583$      & $\bm{128.53}$   \\ \midrule
\multicolumn{2}{|c|}{Failed}    & \multicolumn{2}{c|}{Failed} 
& $\bm{54763}$ & $412.44$       & $54777$      & $\bm{85.38}$    \\ \midrule
\multicolumn{2}{|c|}{Failed}    & \multicolumn{2}{c|}{Failed} 
& $\bm{32647}$ & $245.28$       & $93080$ & $\bm{174.21}$   \\ \midrule
\multicolumn{2}{|c|}{Failed}    & \multicolumn{2}{c|}{Failed} 
& $\bm{44492}$ & $334.85$       & $57919$ & $\bm{86.36}$    \\ \bottomrule
\end{tabular}
\caption{Adding five binary sequences ($n_h = 12$, $n_x =6$).}\label{tab:2c}
\end{subtable}
\caption{Timesteps and the run-times (in seconds) required to achieve $500$ subsequent error-free predictions, i.e., sustainable prediction. The experiments are repeated with five different input streams and the results are presented in order. The bold font shows the best result (in terms of both timestep and run-time) for each input stream. The simulations are performed on a computer with i7-7500U processor, 2.7-GHz CPU, and 8-GB RAM.} 
\end{table*}

In the first part, we evaluate the performance of our algorithm with six real-world datasets. Since the EKF simulations require very high running times, we consider only the first $2500$ input/output pairs in each dataset, i.e., $T = 2500$.

\subsubsection{Kinematic Family of Datasets}

In the first set of experiments, we use the kinematic family of datasets~\cite{Delve}, which are obtained through the simulations of eight-link all revolute robotic arms with different forward-dynamics. The aim is to estimate the distance of the end-effector from a target by using the input predictors, i.e., $n_d = 1$. Here, we consider two datasets from the kinematic family: kin8nm, which is generated with nonlinear dynamics and moderate noise, and kin32fm, which is generated with fairly linear dynamics and moderate noise.

For the kin8nm dataset, we use $8$-dimensional input vectors of the dataset with an additional bias dimension, i.e., $n_x = 9$, and $16$-dimensional state vectors, i.e., $n_h = 16$. In Adam, RMSprop, and SGD, we use the learning rates of $0.004$, $0.009$, and $0.2$ In EKF and DEKF, we choose the initial state covariance matrix as $100 \I$ and anneal the measurement and process noise levels from $10$ to $3$, and $10^{-4}$ to $10^{-6}$, respectively. In Alg2, we initialize the state covariance matrices as $10 \I$ and anneal the process noise level from $10^{-7}$ to $10^{-8}$. 

In Fig.~\ref{fig1:r1}, we demonstrate the temporal loss performance of the compared algorithms for the kin8nm dataset. Here, we observe that Alg2 and EKF provide comparable performances while suffering lower loss values than the other algorithms. In the leftmost column of Table \ref{tab:1a}, we present the confidence intervals of the mean error values and run-times of the algorithms. Here, we observe that as consistent with the temporal performances, SGD provides the worst NSE and kNSE performance, followed by RMSprop, Adam, and DEKF. Moreover, Alg2 and EKF provide approximately $20\%$ improvement in NSE and kNSE compared to RMSprop, Adam, and DEKF while providing very similar error intervals. However, Alg2 provides that equivalent performance in approximately $15$ times smaller run-time. 

For the kin32fm dataset, we use $32$-dimensional input vectors of the dataset with an additional bias dimension, i.e., $n_x = 33$, and $12$-dimensional state vectors, i.e., $n_h = 12$. In Adam, RMSprop, and SGD, we use the learning rates of $0.003$, $0.005$ and $0.4$. In EKF and DEKF, we choose the initial state covariance matrix as $50 \I$ and anneal the measurement and process noise levels from $10$ to $3$, and $10^{-4}$ to $10^{-6}$, respectively.. In Alg2, we initialize the state covariance matrices as $0.5 \I$ and anneal the process noise level from $10^{-7}$ to $10^{-8}$. 

In Fig.~\ref{fig1:r2}, we demonstrate the temporal loss performance of the compared algorithms for the kin32fm dataset. Here, we observe that at the beginning of the training, Alg2 converges to smaller error values faster than the other algorithms. In the rest of the training, all the algorithms except SGD provide similar loss performances, which can be attributable to the relatively linear dynamics of the kin32fm dataset. In the middle column of Table \ref{tab:1a}, we present the confidence intervals and the run-times. Here, we see that EKF and Alg2 provide approximately $35\%$ improvement in NSE and kNSE values compared to DEKF, Adam, and RMSprop.   Differing from the previous experiment, Alg2 achieves slightly better NSE and kNSE performance than EKF. Similar to the previous experiment, Alg2 provides very close performance to EKF in $12$ times smaller run-time.

\subsubsection{Elevators Dataset}

In the second part, we consider the elevators dataset, which is obtained from the controlling procedure of an F16 aircraft \cite{Keel}. Here, the aim is to predict the scalar variable that expresses the actions of the aircraft, i.e., $n_d=1$. For this dataset, we use $18$-dimensional input vectors of the dataset with an additional bias dimension, i.e., $n_x = 19$, and $12$-dimensional state vectors, i.e., $n_h = 12$. In Adam, RMSprop, and SGD, we choose the learning rates as $0.003$, $0.006$, and $0.3$. In EKF and DEKF, we  initialize the state covariance matrix as $ 100 \I$ and anneal the measurement and process noise levels from $10$ to $3$, and $10^{-4}$ to $10^{-6}$. In Alg2, we set the initial state covariance matrices to $10 \I$ and anneal the process noise from $10^{-4}$ to $10^{-8}$. 

In Fig.~\ref{fig2:r1}, we demonstrate the temporal loss performance of the compared algorithms for the elevators dataset. Here, we observe that similar to the previous experiments, Alg2 and EKF provide comparable performances, followed by DEKF, RMSprop, Adam, and SGD. In the rightmost column of Table \ref{tab:1a}, we present the experiment results for the elevators dataset. Here, we see that, as in the previous two experiments, Adam and RMSprop obtain smaller NSE and kNSE values compared to SGD. Furthermore, EKF-based methods, i.e., DEKF, EKF, and Alg2, improve the accuracy of RMSprop and Adam by approximately $30\%$. Here also, EKF and Alg2 achieve the smallest errors while Alg2 reduces the training time of EKF $13$ times without compromising the error performance significantly.

\subsubsection{Pumadyn Family of Datasets}

In the last set of experiments, we use the pumadyn family of datasets obtained through realistic simulations of the dynamics of a Puma $560$ robot arm~\cite{Delve}. The aim is to estimate the angular acceleration of the arm, i.e., $n_d=1$, by using the angular position and angular velocity of the links. In this part, we consider three datasets from the pumadyn family: puma8nm, which is generated with nonlinear dynamics and moderate noise, puma8nh, which is generated with nonlinear dynamics and high noise, and puma32fm, which is generated with fairly linear dynamics and moderate noise

For the puma8nm dataset, we use $8$-dimensional input vectors of the dataset with an additional bias dimension, i.e., $n_x = 9$, and $16$-dimensional state vectors, i.e., $n_h = 16$. In Adam, RMSprop, and SGD, we use the learning rates of $0.006$, $0.01$ and $0.4$. In EKF and DEKF, we initialize the state covariance matrix as $100 \I$ and anneal the measurement and process noise levels from $10$ to $3$, and $10^{-4}$ to $10^{-6}$, respectively. In Alg2, we set the initial state covariance matrices to $50 \I$ and anneal the process noise level from $10^{-4}$ to $10^{-8}$. 

In the puma8nh experiment, we left the dimension of the state and input vectors unchanged. Here, we choose the learning rates of Adam, RMSprop and SGD as $0.006$, $0.008$, and $0.2$. In EKF and DEKF, we initialize the state covariance matrix as $25 \I$ and anneal the measurement and process noise levels from $10$ to $3$, and $10^{-4}$ to $10^{-6}$, respectively. In Alg2, we set the initial state covariance matrices to $10 \I$ and anneal the process noise level from $10^{-7}$ to $10^{-8}$. 

For the puma32fm dataset, we use $32$-dimensional input vectors of the dataset with an additional bias dimension, i.e., $n_x = 33$, and $12$-dimensional state vectors, i.e., $n_h = 12$. In Adam, RMSprop, and SGD, we use the learning rates of $0.003$, $0.004$, and $0.4$. In EKF and DEKF, we choose the initial state covariance matrix as $50 \I$ and anneal the measurement noise and process noise levels from $10$ to $3$, and $10^{-4}$ to $10^{-6}$, respectively. In Alg2, we initialize the state covariance matrices as $0.5 \I$ and  anneal the process noise from $10^{-7}$ to $10^{-8}$.

We plot the learning curves for the puma8nm, puma8nh, and puma32fm datasets in Figures \ref{fig2:r2}, \ref{fig2:r3}, and \ref{fig2:r4}. In the figures, we observe that in all three experiments, EKF and Alg2 provide similar performances while outperforming the other compared algorithms. In Table \ref{tab:1b}, we present the confidence intervals and the run-times. Here, we observe that the error values in the puma8nh experiment are relatively higher due to the nonlinear dynamics and high noise of the puma8nh dataset, whereas the error values in the puma32fm experiment are comparably smaller due to the linearity and moderate noise of the puma32fm dataset. As in the previous parts, Alg2 and EKF provide $10$ to $45\%$ improvements in NSE and kNSE values compared to the other algorithms in all three experiments. However, Alg2 provides the comparable error values in $12$ to $15$ times smaller running times than EKF.

\subsection{Binary Addition}

In this part, we compare the performance of the algorithms on a synthesized dataset that requires learning long-term dependencies. We show that our algorithm learns the long-term dependencies comparably well with EKF, which explains its success in the previous experiments.

To compare the algorithms, we construct a synthesized experiment in which we can control the length of temporal dependence. To this end, we train the network to learn adding $n$ number of binary sequences, where the carry bit is the temporal dependency that the models need to learn. We note that the number of added sequences, i.e., $n$, controls how long the carry bit is propagated on average, hence, the average length of the temporal dependence. We learn binary addition with a purely online approach, i.e., there is only one input stream, and learning continues even when the network makes a mistake.

In the experiments, we consider $n \in \{3, 4, 5\}$. For each $n$, we repeat the experiments with five different input streams, where all input bits are generated randomly by using independent Bernoulli random trials with success probability $0.5$. We assume that the network decides $1$ when its output is positive, and $0$ in vice versa. For the performance comparison, we count the number of symbols needed to attain error-free predictions for $500$ subsequent symbols, i.e., sustainable prediction. We note that as the algorithms demonstrate a faster rate of convergence, the number of steps required to obtain sustainable prediction is expected to be smaller.

In this part, we use $12$-dimensional state vectors, i.e., $n_h = 12$. As the input vectors, we use the incoming $n$ bits with an additional bias, i.e., $n_x = n + 1$.  In EKF and DEKF, we initialize the state covariance matrix as $100 \I$, choose the measurement noise level as $r_t = 3$ for all $t \in [T]$, and anneal the process noise level from $10^{-3}$ to $10^{-6}$. In Alg2, we set all initial state covariance matrices to $10 \I$ and choose the process noise as $q_t = 10^{-7}$ for all $t \in [T]$. In RMSprop, we use the learning rate of $0.02$. Due to space constraints, we compare the EKF-based methods only with RMSprop, which is observed to demonstrate a faster rate of convergence compared to Adam and SGD during the experiments. In the following, we say that an algorithm is failed if it is unable to obtain sustainable prediction after $10^5$ timesteps.

In the first experiment, we consider adding three sequences, i.e., $n = 3$. The experiment results are presented in Table \ref{tab:2a}. Here, we see that Alg2 and EKF achieve $500$ subsequent error-free predictions with considerably fewer data compared to RMSprop and DEKF. On the other hand, RMSprop and DEKF achieve sustainable prediction with shorter running times in general due to their relatively small computational complexity.

In the second and third experiments, we consider adding four and five sequences, i.e., $n = 4$ and $n = 5$, respectively. The results are presented in Table \ref{tab:2b} and Table \ref{tab:2c}. Here, we see that in all experiments, RMSprop and DEKF are unable to obtain subsequent $500$ error-free predictions even after $10^5$ timesteps. On the other hand, EKF and Alg2 consistently achieve the sustainable prediction parallel to their success in the previous experiments. Moreover, we see that EKF requires slightly fewer data to complete the task in general. However, Alg2 requires significantly shorter running times due to its relative efficiency in comparison to EKF.

\section{Conclusion}\label{sec:concl}
We introduce an efficient EKF-based second-order training algorithm for adaptive nonlinear regression using LSTM neural networks. Our algorithm can be used in a broad range of adaptive signal processing applications since it does not assume any underlying data generating process or future information, except for the target sequence being bounded~\cite{lstmdekf, Coskun17}.

We construct our algorithm on a theoretical basis. We first model the LSTM-based adaptive regression problem with a state-space model to derive the update equations of EKF and IEKF. Next, we derive an adaptive hyperparameter selection strategy for IEKF that guarantees errors to converge to a small interval under the assumption that all the data-dependent parameters are known a priori. Finally, we extend our algorithm to a truly online form, where the data-dependent parameters are learned by sequentially observing the data-sequence.

Through an extensive set of experiments, we demonstrate significant performance improvements of our algorithm with respect to the state-of-the-art training methods. To be specific, we show that our algorithm provides a considerable improvement in accuracy with respect to the widely-used adaptive methods Adam~\cite{King14}, RMSprop~\cite{Tieleman12}, and DEKF~\cite{Pusk94}, and very close performance to EKF~\cite{Haykin} with a $10$ to $15$ times reduction in the training time.

As a future work, we consider utilizing our algorithm on other classes of problems, such as sequence to sequence learning or generative models. Another possible research direction is to explore the performance of our algorithm with different sequential learning models, such as Nonlinear Autoregressive Exogenous Models or Hidden Markov Models.

\appendices
\section{} \label{sec:appa}
\begin{proof}[\textbf{Proof of Lemma \ref{lem:statements}}]
In the following, we manipulate the IEKF update rules in (\ref{iekf_1})-(\ref{k_gain}) to obtain the statements in Lemma \ref{lem:statements}. We note that since Algorithm \ref{alg:alg1} performs the IEKF updates if $\normet^2 > 4 \Xo^2$, it guarantees the following statements for  $\{ t : \normet^2 > 4 \Xo^2 \}$.
\begin{enumerate}[wide, labelwidth=!, labelindent=0pt]
\item By multiplying both sides of (\ref{iekf_1}) with $-1$, we write:
\begin{equation}
\label{lem1:11}
-\Thitp =  - \Thit - \Kit (\dt - \dht).
\end{equation}
Then by using (\ref{lstm_ss}), we add $\Titp$ and $\Tit$ to both sides of (\ref{lem1:11}) respectively:
\begin{equation}
\Titp -\Thitp = (\Tit - \Thit) - \Kit (\dt - \dht).
\end{equation}
By using the Taylor series expansion in (\ref{eq:taylor2}) and using the notation $\zit= (\Tit - \Thit)$, we write
\begin{align}
\label{eq:ed3}
\zitp = \zit - \Kit \Hit \zit - \Kit \Xit.
\end{align}
The statements in (\ref{eq:ed1}) and (\ref{eq:ed2}) follow (\ref{eq:ed3}).
\item  By (\ref{ind:cov_upd}), for all $i \in [(4 n_s+n_d)]$,
\begin{align}
\Pitp&- q_t \I = ( \I - \Kit \Hit) \Pit \label{lem1:21}\\
&= \big(  \I - \Pit \Hit^T ( \Hit \Pit \Hit^T + r_{i,t} \I  )^{-1} \Hit \big) \Pit \label{lem1:22} \\
&=  \Pit - \Pit \Hit^T ( \Hit \Pit \Hit^T + r_{i,t} \I )^{-1} \Hit \Pit \label{lem1:23}
\end{align}
where we use the formulation of $\Kit$ in (\ref{k_gain}) to write (\ref{lem1:21}) from (\ref{lem1:22}). By applying the matrix inversion lemma\footnote{Matrix inversion lemma: $\Matrixinversionlemma$.} to (\ref{lem1:23}), we write
\begin{equation}
\label{lem1:24}
(\Pitp - q_t \I )^{-1}= \Pit^{-1} +  r_{i,t}^{-1} \Hit^T  \Hit.
\end{equation}
By noting that $\textbf{P}_{i,1}^{-1} > \textbf{0}$, and using (\ref{lem1:24}) as the induction hypothesis, it can be shown that $(\Pitp - q_t \I )^{-1}$ exists and  $(\Pitp - q_t \I )^{-1}>\textbf{0}$, for all $t \in [T]$. Since $q_t \geq 0$,  $\Pit^{-1}$ has the same properties, which leads to (\ref{eq:pthi2}). Also,  (\ref{eq:pthi1}) can be reached by taking the inverse of  both sides in (\ref{lem1:21}).
\item By multiplying both sides of (\ref{eq:ed1}) with $(\Pitp - q_t \I )^{-1}$, and using (\ref{eq:pthi1}), (\ref{eq:pt_ed}) can be obtained.
\end{enumerate}
\end{proof}

\begin{proof}[\textbf{Proof of Thorem \ref{th:main}}]
To prove Theorem \ref{th:main}, we use the second method of Lyapunov. Let us fix an arbitrary node $i \in [(4n_s+n_d)]$, and choose the Lyapunov function as
\begin{equation}
V_{i,t}= \zit^T \Piti \zit.
\end{equation}
Let us say that $\Delvt= V_{i,t+1}-V_{i,t}$. Since we update $\Pit$, and $\Thit$ only when $\normet^2 > 4 \Xo^2$, for  $\{ t : \normet^2 \leq 4 \Xo^2 \}$, $\Delvt=0$. Therefore in the following, we only consider the time steps, where we  perform the weight update, i.e, $\{ t : \normet^2 > 4 \Xo^2 \}$. 	

To begin with, we write the open formula of $\Delvt$:
{\small
\begin{align}
&\Delvt= \zitp^T \Pitpi \zitp - \zit^T \Piti \zit  \\
&\leq \zitp^T (\Pitp - q_t \I )^{-1} \zitp - \zit^T \Piti \zit  \label{eq:last0125} \\
&= \zitp^T \Piti \zit  - \zitp^T  (\Pitp - q_t \I )^{-1} \Kit \Xit  - \zit^T \Piti \zit \label{eq:last025} \\
&= (\zitp-\zit)^T \Piti \zit - \zitp^T  (\Pitp - q_t \I )^{-1} \Kit \Xit  \\
&= (- \Kit \Hit \zit - \Kit \Xit)^T \Piti \zit \nonumber \\ 
&\quad - \zitp^T (\Pitp - q_t \I )^{-1} \Kit \Xit  \label{eq:last05} \\
&= - \zit^T \Hit^T \Kit^T \Piti \zit - \Xit^T \Kit^T \Piti \zit \nonumber \\
&\quad - \zitp^T  (\Pitp - q_t \I )^{-1} \Kit \Xit, \label{eq:last1}
\end{align} }%
where we use the 2nd statement in Lemma \ref{lem:statements} for (\ref{eq:last0125}), (\ref{eq:pt_ed}) for (\ref{eq:last025}), and (\ref{eq:ed2}) for (\ref{eq:last05}). For the sake of notational simplicity, we introduce $\Mit=\Mitf$, where $\Kit= \Pit \Hit^T \Miti$. Then, we write (\ref{eq:last1}) as
\begin{align}
\Delvt &\leq - \zit^T \Hit^T \Miti \Hit \zit  - \Xit^T \Miti \Hit \zit \nonumber \\
&\quad - \zitp^T (\Pitp - q_t \I )^{-1} \Kit \Xit. \label{eq:last2}
\end{align}
We write the last term in (\ref{eq:last2}) as
{\small
\begin{align}
- &\zitp^T (\Pitp - q_t \I )^{-1} \Kit \Xit \nonumber \\
&= - \Xit^T \Kit^T (\Pitp - q_t \I )^{-1} \zitp   \\
&= - \Xit^T \Kit^T \Big( \Piti \zit - (\Pitp - q_t \I )^{-1} \Kit \Xit \Big) \label{eq:last25}   \\
&=  - \Xit^T \Kit^T \Piti \zit + \Xit^T \Kit^T  (\Pitp - q_t \I )^{-1} \Kit \Xit \\
&= - \Xit^T \Miti \Hit \zit + \Xit^T \Kit^T  (\Pitp - q_t \I )^{-1} \Kit \Xit, \label{eq:last3}
\end{align} }%
where we use (\ref{eq:pt_ed}) to obtain (\ref{eq:last25}). By (\ref{eq:last3}) in (\ref{eq:last2}), we write
\begin{align}
\Delvt \leq &- \zit^T \Hit^T \Miti \Hit \zit - 2  \Xit^T \Miti \Hit \zit  \nonumber  \\
& + \Xit^T \Kit^T  (\Pitp - q_t \I )^{-1} \Kit \Xit. \label{eq:last4}
\end{align}
We add  $\pm \Xit^T \Miti \Xit$ to (\ref{eq:last4}), and group the terms as
\begin{align}
\Delvt \leq  & \Xit^T \Big( \Kit^T (\Pitp - q_t \I )^{-1} \Kit + \Miti \Big) \Xit  \nonumber \\
 &-(\Hit \zit + \Xit)^T \Miti (\Hit \zit + \Xit). \label{eq:last5}
\end{align}
By using (\ref{eq:pthi2}), the definition of $\et$ in (\ref{eq:taylor2}), and formulation of $\Kit$, we write (\ref{eq:last5}) as
{\small 
\begin{align}
\Delvt &\leq   \Xit^T \Big( \Miti \Hit \Pit (\Piti +  r_{i,t}^{-1} \Hit^T \Hit) \Pit \Hit^T \Miti  \nonumber\\ 
&\qquad \qquad + \Miti \Big) \Xit - \et^T \Miti \et  \nonumber  \\
&=\Xit^T \Big(  r_{i,t}^{-1} \Miti \Hit  \Pit \Hit^T \Hit \Pit \Hit^T \Miti \nonumber \\
&\qquad \qquad   + \Miti \Hit \Pit \Hit^T \Miti    + \Miti \Big) \Xit- \et^T \Miti \et. \label{eq:last6}
\end{align} }
Note that since $\Mit^{-1} = \Mitf^{-1}$,  $\Miti \leq r_{i,t}^{-1} \I $, and $\Hit \Pit \Hit^T \Miti \leq \I$. By using these two inequalities in (\ref{eq:last6}), we get
\begin{equation}
\Delvt \leq  \frac{3 \norm{\Xit}^2}{r_{i,t}} - \frac{\norm{\et}^2}{\max_{j \in [n_\theta]} \lambda_j(\Mit)}, \label{eq:last7}
\end{equation}
where $\lambda_j(\Mit)$ denotes the $\textrm{j}^\textrm{th}$ eigenvalue of $\Mit$. We note that  $\Tr(\Mit)/n_d \leq \max_{j \in [n_\theta]} \lambda_j(\Mit)$, and $\Tr(\Mit)= \Tr(\Hit \Pit \Hit^T) + n_d r_{i,t}$. Then, by using $\norm{\Xit}^2 \leq \Xo^2$ and  (\ref{eq:last7}), we write
\begin{equation}
\Delvt \leq  \frac{ 3 \Xo^2}{r_{i,t}} - \frac{n_d \norm{\et}^2}{\Tr(\Hit \Pit \Hit^T) + n_d r_{i,t}}. \label{eq:last8}
\end{equation}
To ensure stability, we need $\Delvt < 0$. For this, it is sufficient to guarantee that the right hand side of (\ref{eq:last8}) is smaller than $0$, i.e.,
\begin{equation}
\Delvt \leq  \frac{ 3 \Xo^2}{r_{i,t}} - \frac{n_d \norm{\et}^2}{\Tr(\Hit \Pit \Hit^T) + n_d r_{i,t}} <0. \label{eq:last85}
\end{equation}
To guarantee (\ref{eq:last85}), we need to ensure
\begin{equation}
\frac{3 \Tr(\Hit \Pit \Hit^T) \Xo^2 / n_d}{\norm{\et}^2- 3 \Xo^2} < r_{i,t}. \label{eq:last875}
\end{equation}
Since we only consider $\{ t : \normet^2 > 4 \Xo^2 \}$, we can bound  the left hand side of (\ref{eq:last875}) as
\begin{align}
\frac{3 \Tr(\Hit \Pit \Hit^T) \Xo^2 / n_d}{\norm{\et}^2- 3 \Xo^2} &< \frac{3 \Tr(\Hit \Pit \Hit^T) \Xo^2 / n_d}{4 \Xo^2- 3 \Xo^2}. \\
&= 3 \Tr(\Hit \Pit \Hit^T) / n_d \label{eq:last9}
\end{align}
Therefore $r_{i,t}=3 \Tr(\Hit \Pit \Hit^T) / n_d$ ensures stability. 

As we ensure $\Delvt < 0$ for $\{ t : \normet^2 > 4 \Xo^2 \}$  and $\Delvt = 0$ for $\{ t : \normet^2 \leq 4 \Xo^2 \}$ ,  Algorithm \ref{th:main} guarantees
\begin{equation}
\label{eq:last10}
\zit^T \Piti \zit \leq  \Vb =  p_1^{-1} \norm{\zib}^2,
\end{equation}
for all $t \in [T]$. Since $\norm{\zib}$ is finite, as a result of (\ref{eq:last10}), $\zit^T \Piti \zit$ is also finite. Under the condition that $\Tr(\Pit)$ stays bounded, $\norm{\zit}$ should stay bounded, which proves the 1st statement of Theorem \ref{th:main}.  Moreover, since $\Delvt < 0$ for  $\{ t : \normet^2 > 4 \Xo^2 \}$, as the cardinality of $\{ t : \normet^2 > 4 \Xo^2 \}$ approaches to infinity, $\zit^T \Piti \zit$ approaches to $0$, i.e., $\zit^T \Piti \zit \to 0$. If $\Tr(\Pit)$ stays bounded, $\zit^T \Piti \zit \to 0$ implies that  $\norm{\zit} \to 0$.  Since we learn the LSTM parameters $\Tht$ only  when $\{ t : \normet^2 > 4 \Xo^2 \}$, by the desired data model in (\ref{lstm_ss})-(\ref{lstm_ss1}), Algorithm \ref{alg:alg1} converges to the weights, which guarantees a loss value lower than or equal to $4 \Xo^2$. This proves the 2nd statement of Theorem \ref{th:main}.
\end{proof}

\begin{proof}[\textbf{Proof of Thorem \ref{th:bound}}]
To prove the statement in the theorem, we construct the scenario that maximizes $\Xo$ for $\Thb \approx \textbf{0}$ and show that $\Xo \in [0, \sqrt{n_d}]$ in this scenario. We construct this scenario in three steps:
\begin{enumerate}[leftmargin=*]
\item  Due to our LSTM model in (\ref{lstm_1})-(\ref{lstm_7}),  as $\norm{\Tht}$ approaches to $0$,  the entries of $\dht$ and $\Ht$ approaches to $0$ as well, i.e.,  as $\norm{\Tht} \to 0$,  $\dht \to \textbf{0}$ and $\Ht \to \textbf{0}$. Hence, by (\ref{eq:taylor2}), as  $\norm{\Tht} \to 0$, $\Xit \to \et$ for all $t \in [T]$ and $i \in [( 4 n_s + n_d)]$.
\item  Let us assume $\Xo \leq \sqrt{n_d}$, which we will validate in the following, and use $\Xo= \sqrt{n_d}$. Since $\Xo = \sqrt{n_d}$ does not allow any weight update, $\norm{\Thb} \leq \epsilon $ and $\Xo =\sqrt{ n_d}$ result in $\norm{\Tht} \leq \epsilon$ for all $t \in [T]$, where $\epsilon$ is a small positive number. 
\item By the 1st step, $\textbf{d}_t= \pm \textbf{1}$ maximizes $\Xo$ in the case of $\norm{\Tht} \leq \epsilon$. By the 2nd step, if $\norm{\Thb} \leq \epsilon$ and $\Xo= \sqrt{n_d}$, then $\norm{\Tht} \leq \epsilon$. This implies that as $\norm{\Thb} \to 0$, $\norm{\Xit} \to \norm{\et}= \sqrt{n_d}$. Thus, our assumption in the 2nd step holds at some point for any bounded data sequence. 
Therefore, we can choose sufficiently small initial weights to ensure $\Xo \in [0,\sqrt{n_d}]$.
\end{enumerate} 
\end{proof}
\section{} \label{sec:appb}
For the proof, we use the notion of exp-concavity, which is defined in \cite[Definition 2]{volkanexp}. Note that by using \cite[Definition 2]{volkanexp}, we can show that $F(\dht)= \norm{\dt - \dht}^2$, where $\dt,\dht \in [-1,1]^{n_d}$, is $\frac{1}{8 n_d}$-exp-concave. 
\begin{proof}[\textbf{Proof of Thorem \ref{th:main2}}]
We note that since we assume $\Xo_\textrm{best} \in [\Xo_\textrm{min}, \sqrt{n_d}]$, there is an Algorithm \ref{alg:alg1} instance in Algorithm \ref{alg:alg2}, which uses  $\Xo \in [\Xo_\textrm{best}, 2 \Xo_\textrm{best}]$. Let us use  $\tilde{j}$ for the index of that instance, and $\ell_{j,t}$ to denote the loss of the $\textrm{j}^\textrm{th}$ Algorithm \ref{alg:alg1} instance at time $t$, i.e, $\ell_{j,t}= \norm{\textbf{d}_t-\bm{\hat{d}}_{j,t}}^2$. Let us also use  $\ell_{A,t}$ to denote the loss of Algorithm \ref{alg:alg2} at time $t$, where  $\ell_{A,t}= \norm{\dt-\dht}^2$.  In this proof, our aim is to show that $\lim_{t \to \infty} \ell_{A,t} - \ell_{\tilde{j},t}=0$. For the proof, we introduce three shorthand notations:
\begin{equation*}
W_t= \sum_{j=1}^N w_{j,t}, \qquad L_{j,T}= \sum_{t =1 }^T \ell_{j,t}, \qquad  L_{A,T}= \sum_{t =1 }^T \ell_{A,t}. 
\end{equation*}
To begin with, we find a lower bound for  $\ln (W_{T+1}/ W_1)$:
\begin{align}
\ln \frac{W_{T+1}}{W_1} &= \ln \Big( \sum_{j=1}^N \exp(- \frac{1}{8 n_d} L_{j,T}) \Big) - \ln N \\
&\geq  - \frac{1}{8 n_d}  L_{\tilde{j},T} - \ln N. \label{p2:last1}
\end{align}
Then, we find an upper bound for $\ln (W_{t+1}/W_t)$:
\begin{align}
\ln \frac{W_{t+1}}{W_t} &= \ln \Big( \frac{\sum_{j=1}^N w_{j,t} \exp( - \frac{1}{8 n_d} \norm{\dt - \dhjt}^2)}{\sum_{k=1}^N w_{k,t}} \Big) \\
&\leq \ln \Big( \exp(  - \frac{1}{8 n_d} \norm{\dt - \frac{\sum_{j=1}^N w_{j,t} \dhjt}{\sum_{k=1}^N w_{k,t}}}^2 ) \Big) \label{p2:last15} \\
&= - \frac{1}{8 n_d}  \norm{\dt-\dht}^2  = -  \frac{1}{8 n_d}  \ell_{A,t}, \label{p2:last2}
\end{align}
where we use the exp-concavity of $F(\dht)= \norm{\dt - \dht}^2$ and Jensen's inequality for (\ref{p2:last15}). Since $\sum_{t=1}^T \ln \frac{W_{t+1}}{W_t}= \ln \frac{W_{T+1}}{W_1}$, by using (\ref{p2:last1}) and (\ref{p2:last2}), we write
\begin{align}
 - \frac{1}{8 n_d}  L_{\tilde{j},T} - \ln N \leq - \frac{1}{8 n_d}\sum_{t=1}^T \ell_{A,t},
\end{align}
which can be equivalently written as 
\begin{align}
L_{A,T} -  L_{\tilde{j},T}  \leq 8 n_d \ln N. \label{p2:last3}
\end{align}
Note that by (\ref{eq:last3}) the series $L_{A,T} -  L_{\tilde{j},T}$  is convergent, which means $\lim_{t} \ell_{A,t} - \ell_{\tilde{j},t} = 0$. Since the Algorithm \ref{alg:alg1} instance with index $\tilde{j}$ has $\Xo \in [\Xo_\textrm{best}, 2 \Xo_\textrm{best}]$, by Theorem \ref{th:main}, the loss sequence of Algorithm \ref{alg:alg2} converges to $[0, 16 \Xo_\textrm{best}^2]$.

\end{proof}

\bibliographystyle{IEEEtran}
\balance
\bibliography{my_references}

\end{document}